\documentclass[10pt,twocolumn,letterpaper]{article}

\usepackage[pagenumbers]{cvpr} 

\usepackage{amssymb}
\usepackage{pifont}
\usepackage{nicematrix}
\usepackage{algorithm}
\usepackage{algorithmic}
\usepackage{amsmath}
\usepackage{bm}
\usepackage{graphicx}
\usepackage{array}
\usepackage{booktabs}
\usepackage[table]{xcolor}
\usepackage{multirow}
\usepackage{arydshln} 
\usepackage{makecell}
\usepackage{dblfloatfix}
\usepackage{subcaption}
\usepackage{cuted}
\usepackage{caption}
\usepackage[normalem]{ulem}
\definecolor{link}{RGB}{214,51,132}

\newtheorem{proposition}{Proposition}

\newtheorem{proof}{Proof}
\newtheorem{remark}{Remark}
\newtheorem{lemma}{Lemma}

\definecolor{cvprblue}{rgb}{0.21,0.49,0.74}
\usepackage[pagebackref,breaklinks,colorlinks,allcolors=cvprblue]
{hyperref}

\def\paperID{*****} 
\def\confName{CVPR}
\def\confYear{2026}

\title{%
\vspace{-15pt}
  \makebox[\linewidth][l]{%
    \raisebox{-0.5\height}{\includegraphics[height=1.2em]{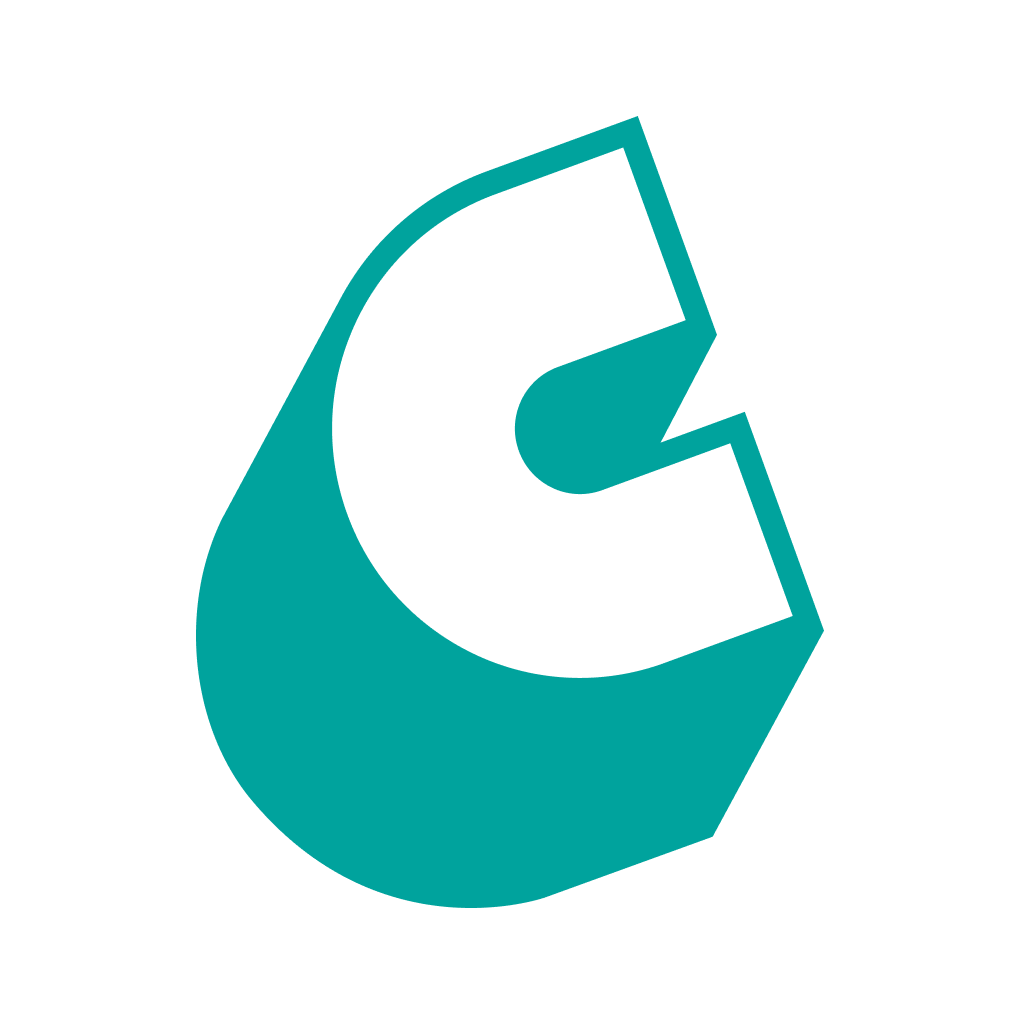}}%
    \hspace{-0.35em}%
    \raisebox{-0.5\height}{\includegraphics[height=1.5em]{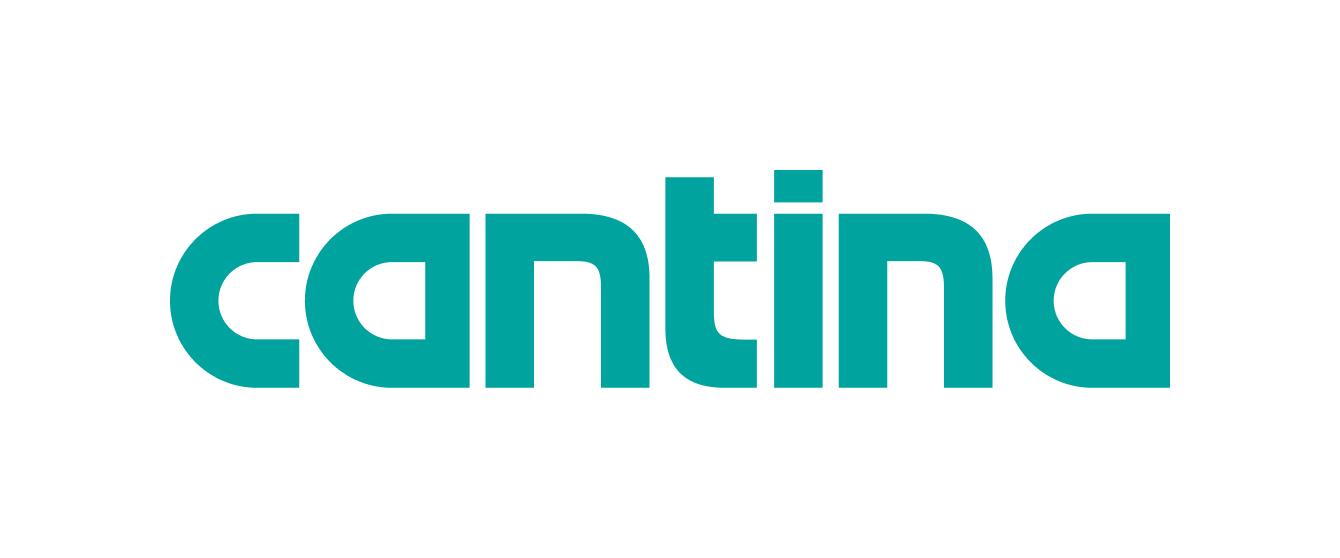}}%
  }%
  \par\vspace{-6pt}%
  \noindent{\color{black!40}\rule{\linewidth}{0.8pt}}%
  \par\vspace{2pt}%
  \raisebox{-0.28\height}{\includegraphics[height=1.7em]{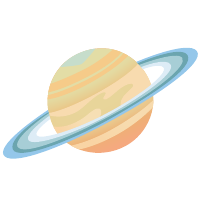}}\,OrbitQuant: Data-Agnostic Quantization for Image and \\ Video Diffusion Transformers\vspace{-10pt}}

\author{Donghyun Lee$^{1,2,\dagger}$\thanks{Work done during an internship at Cantina Labs.} , Jitesh Chavan$^1$, Duy Nguyen$^{1,3}$, Sam Huang$^1$, Liming Jiang$^1$, \\ Priyadarshini Panda$^2$, Timo Mertens$^1$, Saurabh Shukla$^{1,\dagger}$\\[3pt]
$^1$Cantina Labs, $^2$University of Southern California, $^3$University of Illinois Urbana-Champaign
\\[2pt]{\footnotesize $^{\dagger}$Correspondence to: saurabh@cantina.ai, donghyun.lee.1@usc.edu}
\\[2pt]{\normalsize Project Page: {\tt \href{https://saurabhcantina.github.io/orbitquant/}{\textcolor{link}{https://saurabhcantina.github.io/orbitquant/}}}}}

\makeatletter
\patchcmd{\@maketitle}{\vskip .375in}{\vskip .12in}{}{}
\makeatother

\begin{document}
\maketitle

\begin{strip}
\centering
\vspace{-20mm}
\includegraphics[width=\linewidth]{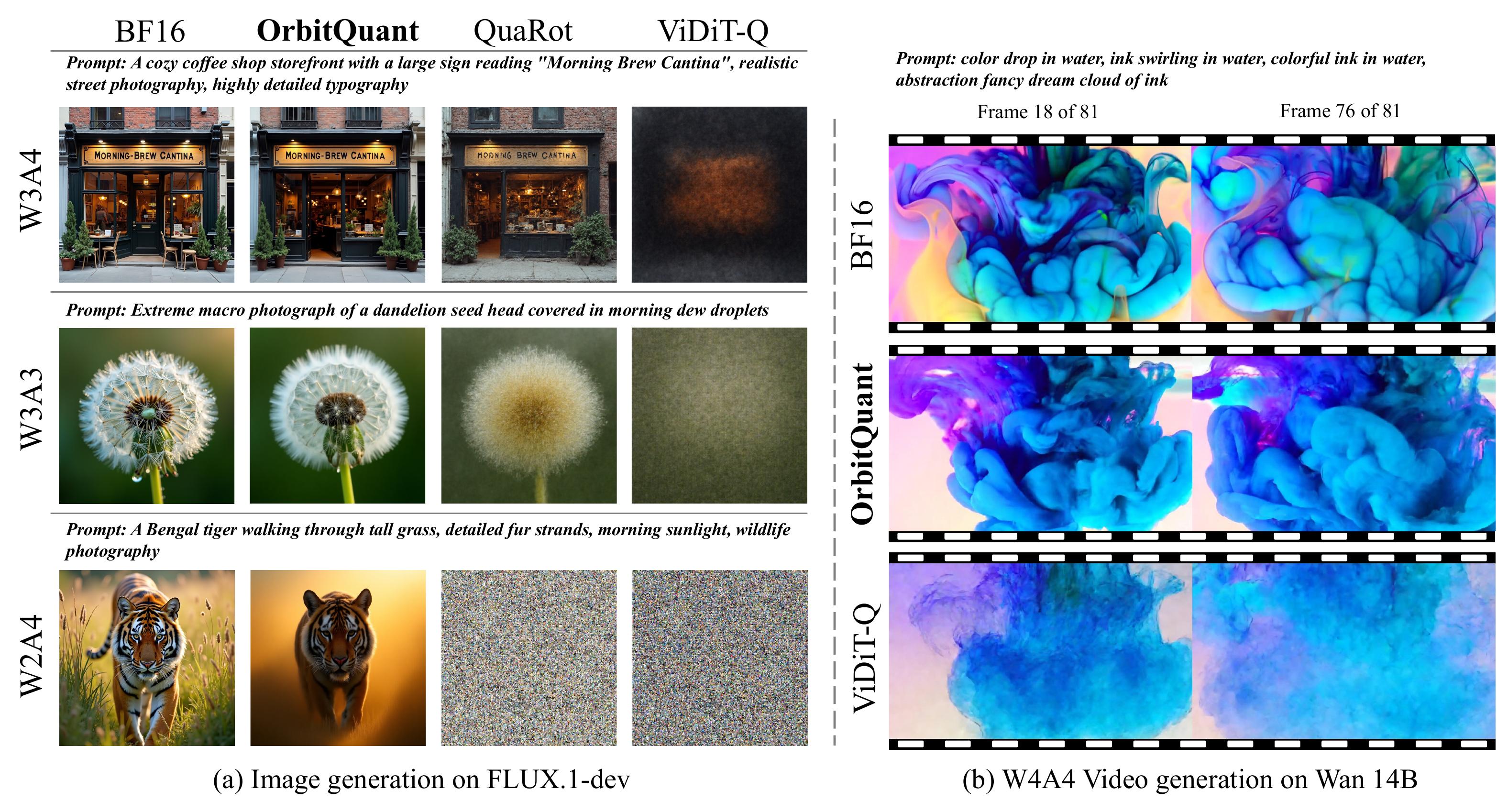}
\vspace{-8mm}
\captionof{figure}{Qualitative comparison of OrbitQuant against QuaRot~\cite{ashkboos2024quarot} and ViDiT-Q~\cite{zhao2024vidit} under low-bit quantization, with the BF16 full-precision output. (a) Image generation on FLUX.1-dev at W3A4, W3A3, and W2A4. (b) Video generation on Wan 14B at W4A4.}
\vspace{-3mm}
\label{fig:teaser}
\end{strip}


\begin{abstract}
Diffusion transformers (DiTs) achieve state-of-the-art image and video generation, but their multi-step sampling and growing parameter count make inference expensive. Post-training quantization (PTQ) is the natural remedy, yet DiT activations shift across timesteps, prompts, and guidance branches, forcing prior methods to re-fit calibration data for every new checkpoint or modality. We present OrbitQuant, a data-agnostic weight-activation quantizer that bypasses range estimation by quantizing in a normalized, rotated basis. In this basis, a randomized permuted block-Hadamard (RPBH) rotation concentrates each coordinate around one fixed, known marginal regardless of the input, so a single Lloyd–Max codebook serves all timesteps, prompts, and layers of a given input dimension. We extend the same quantizer to weight rows offline, absorbing the rotation into the weights so that it cancels inside each linear layer and only a forward rotation on the activations remains at runtime. The same recipe transfers from image to video with no per-modality tuning. Across FLUX.1, Z-Image-Turbo, Wan 2.1, and CogVideoX, it sets the state of the art for PTQ at several low-bit settings. It also pushes PTQ of image diffusion transformers to W2A4 with usable generation quality.
\end{abstract}
\vspace{-6mm}

\vspace{-2mm}
\section{Introduction}
\label{sec:intro}
\vspace{-2mm}
Diffusion models have become the dominant paradigm for high-fidelity image and video generation. Traditionally, these models employ a convolutional U-Net as the denoising backbone~\cite{rombach2022high, saharia2022photorealistic, singer2022make, ho2020denoising}. More recently, however, the field has shifted to transformer-based denoisers. Diffusion Transformers (DiTs)~\cite{peebles2023scalable, bao2023all} replace the U-Net with a stack of attention blocks that scales favorably with model size and data, and now underpin state-of-the-art image~\cite{chen2024pixart, esser2024scaling, labs2025flux1kontextflowmatching, xie2024sana, wu2025qwen, cai2025z} and video generators~\cite{sora2024, yang2025cogvideox, kong2024hunyuanvideo, wan2025wan}.

Despite their quality, DiTs are expensive to run at inference for two reasons. First, the transformer trunk is evaluated repeatedly across many sequential denoising timesteps. Second, unlike LLM decoding, where latency is dominated by weight loading~\cite{lin2024awq, frantar2022gptq}, DiT inference is compute-bound even at a single batch, so weight-only quantization yields no measured speedup~\cite{li2024svdquant}. Low-bit post-training quantization (PTQ) of both weights and activations is therefore the natural remedy, since it compresses both the memory footprint and the compute of every step without any retraining.

PTQ is most mature for large language models (LLMs), where activation outliers are handled by scaling them into the weights~\cite{xiao2023smoothquant} or rotating them away~\cite{ashkboos2024quarot, liu2025spinquant, chee2023quip, tseng2024quip}. Both assume activation statistics that a single calibration pass can capture, which holds for LLMs but breaks for diffusion transformers. DiT activations exhibit channel-wise outliers~\cite{wu2024ptq4dit} and shift across timesteps, prompts, and classifier-free-guidance branches~\cite{chen2025q, zhao2024vidit}. Existing DiT PTQ methods absorb this drift with calibration~\cite{zhao2024vidit, li2024svdquant, zhang2026adatsq}, so each new checkpoint, resolution, or modality requires a calibration set to be re-collected and re-fit.


We propose \textbf{OrbitQuant}, a rotation-based PTQ framework for diffusion transformers. A DiT activation offers no stable range to calibrate against, since it moves with every input. Rather than chase that moving target with per-input scales, OrbitQuant rotates it away. A random rotation turns a normalized activation into coordinates that follow one fixed, known distribution regardless of the input~\cite{zandieh2025turboquant}, so a single Lloyd--Max codebook built offline quantizes every activation and is shared across all denoising steps. OrbitQuant realizes this as a randomized permuted block-Hadamard (RPBH) rotation, and we find that a uniform random permutation suffices to keep the rotated marginal well-behaved at low bit-width on DiT activations. The same rotation is folded into the weight rows offline, so it cancels inside each linear layer, with weights and activations quantized in one shared basis, leaving only a single forward RPBH rotation at inference. The main contributions of our work are as follows:

\begin{itemize}
    \item We cast low-bit DiT activation quantization as a distributional codebook problem, replacing per-timestep range calibration with a single Lloyd--Max codebook fit to a fixed post-rotation marginal and shared across all denoising steps.
    \item We extend the same quantizer to the weight rows with a shared-rotation design that quantizes weights and activations in one common basis.
    \item We propose the RPBH rotation, an efficient rotation whose uniform random permutation keeps activations well-quantizable at low bit-width without calibration.
    \item We evaluate OrbitQuant on image and video DiTs, achieving state-of-the-art PTQ on GenEval and VBench without calibration data. At W2A4, where prior PTQ baselines collapse to noise, OrbitQuant is the only method that still produces usable images, shown in Figure~\ref{fig:teaser}.
\end{itemize}

\begin{figure*}[t]
\centering
\includegraphics[width=\linewidth]{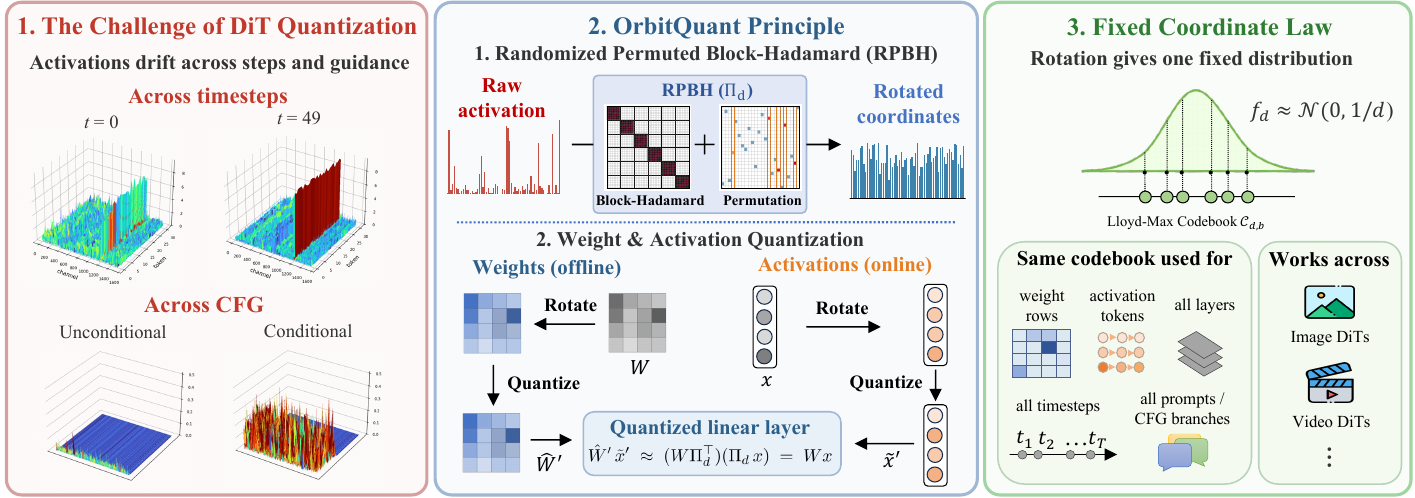}
\vspace{-6mm}
\caption{\textbf{Overview of OrbitQuant.} \textbf{(1)} DiT activations drift across timesteps and CFG branches, so calibrated scales do not transfer. \textbf{(2)} The RPBH rotation $\Pi_d$ maps raw activations to well-behaved coordinates. Folded into the weights, it cancels inside each layer ($\hat{W}'\hat{x}' \approx Wx$). \textbf{(3)} Rotated coordinates concentrate around one fixed marginal $f_d \approx \mathcal{N}(0,1/d)$, so a single Lloyd--Max codebook $\mathcal{C}_{d,b}$ per dimension serves all layers, timesteps, prompts, and both image and video DiTs, with no calibration.}
\vspace{-3mm}
\label{fig:overview}
\end{figure*}

\vspace{-2mm}
\section{Related Work}
\label{sec:related}
\vspace{-2mm}

\paragraph{LLM Quantization.}
In LLMs, Weight-only methods quantize weights alone and suit the memory-bound decoding regime~\cite{frantar2022gptq, lin2024awq, kim2023squeezellm, dettmers2023spqr}, while quantizing activations requires handling outlier channels, either by scaling them into the weights~\cite{xiao2023smoothquant} or by rotating them away. Rotation-based methods~\cite{ashkboos2024quarot, liu2025spinquant, chee2023quip, tseng2024quip, sun2024flatquant, hu2025ostquant} fold a Hadamard or learned rotation into the weights by computational invariance, leaving the output unchanged while making activations easy to quantize. 

Recent block rotations pair the rotation with a calibrated permutation. DuQuant~\cite{lin2024duquant} orders channels by outlier magnitude, and PeRQ~\cite{sanjeet2026pushing} fits a permutation that balances per-block mass, which its analysis shows governs block-Hadamard outlier suppression. Random rotations also enable calibration-free vector quantization. PolarQuant~\cite{han2026polarquant} quantizes KV embeddings in polar coordinates, building a codebook from the analytically known angle distribution after random preconditioning. TurboQuant~\cite{zandieh2025turboquant} brings this distributional codebook to Cartesian coordinates with a dense Haar rotation and a Beta-marginal Lloyd--Max codebook. Both are standalone KV-cache vector quantizers that rotate back to reconstruct. OrbitQuant instead applies the rotation-plus-codebook idea inside DiT projections, where the shared rotation cancels rather than being inverted, and replaces the dense Haar with an efficient RPBH rotation. Unlike prior permuted rotations, it draws the permutation uniformly at random, with a probabilistic guarantee that the rotated coordinates stay well-behaved.

\vspace{-2mm}
\paragraph{Diffusion Quantization.}
Most existing DiT quantization methods are calibration-based. SVDQuant~\cite{li2024svdquant} absorbs activation outliers with a high-precision low-rank branch fit on a calibration set, PTQ4DiT~\cite{wu2024ptq4dit} balances salient channels with block reconstruction, AdaTSQ~\cite{zhang2026adatsq} fits per-channel scales with timestep-sensitive precision allocation, and ViDiT-Q~\cite{zhao2024vidit} pairs per-channel calibration with mixed precision on both image and video DiTs. LRQ-DiT~\cite{yang2025lrq} adds calibrated DuQuant-style~\cite{lin2024duquant} rotations on outlier-heavy layers, PermuQuant~\cite{cheng2026permuquant} calibrates channel reordering for per-group quantization, and S$^2$Q-VDiT~\cite{feng2025s} selects calibration data by Hessian-aware saliency with token-level distillation on video DiTs. QVGen~\cite{huang2025qvgen} and RobuQ~\cite{yang2025robuq} depart from PTQ with quantization-aware training (QAT), the latter reaching ternary weights on ImageNet DiTs. Closer to our setting, DVD-Quant~\cite{li2025dvd} is data-free, pairing a rotated quantizer with grid refinement and adaptive bit allocation, but it is tailored to video DiTs with per-model machinery, and ConvRot~\cite{huang2025convrot} pairs calibration-free group-wise regular Hadamard rotations with a uniform grid on FLUX. In contrast, OrbitQuant uses a fully analytic distribution-derived codebook that requires no model evaluation at quantizer construction, and transfers unchanged between image and video DiTs.

\vspace{-2mm}
\section{Preliminaries}
\label{sec:prelim}
\vspace{-1mm}
This section fixes notation and reviews the two ingredients we inherit from TurboQuant~\cite{zandieh2025turboquant}, namely a Haar-random orthogonal rotation and a Lloyd--Max scalar codebook designed against the post-rotation coordinate distribution.

\subsection{Notation}
\label{sec:prelim:notation}
\vspace{-0.9mm}
We write matrices in bold uppercase (e.g., $\mathbf{W}$), vectors in bold lowercase (e.g., $\mathbf{x}$), and scalars in plain type. A DiT block is built from linear projections
\begin{equation}
\mathbf{y} = \mathbf{W}\mathbf{x}, \quad \mathbf{W} \in \mathbb{R}^{m \times d},\ \mathbf{x} \in \mathbb{R}^d,
\label{eq:linear}
\end{equation}
applied token-wise to image- or text-token streams. We write $\mathbf{w}_i^\top \in \mathbb{R}^d$ for the $i$-th row of $\mathbf{W}$ and $r_i = \|\mathbf{w}_i\|_2$ for its $\ell_2$ norm. Given weight and activation bit-widths $b_w$ and $b_a$, our goal is to replace $\mathbf{W}$ and $\mathbf{x}$ with quantized surrogates $\hat{\mathbf{W}}$ and $\hat{\mathbf{x}}$ at $b_w$ and $b_a$ bits per coordinate, so that $\hat{\mathbf{W}}\hat{\mathbf{x}} \approx \mathbf{W}\mathbf{x}$ at every denoising step and for every prompt, without calibration data. We write $\mathcal{L}$ for the set of target linear layers and $\mathcal{D}$ for the distinct input dimensions in $\mathcal{L}$.

\subsection{TurboQuant}
\label{sec:prelim:turboquant}
\vspace{-1mm}
TurboQuant~\cite{zandieh2025turboquant} is a calibration-free vector quantizer, originally for KV-cache compression, that quantizes a vector in two steps. First, it normalizes $\mathbf{x}$ to $\tilde{\mathbf{x}} = \mathbf{x} / \|\mathbf{x}\|_2$, keeps the norm, and applies a Haar-random orthogonal rotation $\boldsymbol{\Phi}_d \in \mathbb{R}^{d \times d}$~\cite{mezzadri2006generate}. Regardless of $\mathbf{x}$, each coordinate of $\boldsymbol{\Phi}_d \tilde{\mathbf{x}}$ then follows the fixed marginal
\begin{equation}
f_d(t) = \frac{\Gamma(d/2)}{\sqrt{\pi}\,\Gamma((d-1)/2)} (1 - t^2)^{(d-3)/2}, \quad t \in [-1, 1],
\label{eq:beta-marginal}
\end{equation}
where $\Gamma(\cdot)$ is the Gamma function. For $d \geq 64$, this marginal is tightly approximated by $\mathcal{N}(0, 1/d)$, and distinct coordinates are nearly independent. Second, since $f_d$ is known offline, we precompute an MSE-optimal Lloyd--Max codebook~\cite{lloyd1982least, max1960quantizing} for each $(d, b) \in \mathcal{D} \times \{b_w, b_a\}$, giving $2^b$ centroids $\mathcal{C}^{(d,b)} = \{c_1^{(d,b)}, \ldots, c_{2^b}^{(d,b)}\}$ and the nearest-centroid map
\begin{equation}
\hat{q}_b^{(d)}(t) = \operatorname*{arg\,min}_{c \,\in\, \mathcal{C}^{(d,b)}} |t - c|,
\label{eq:lloyd-max}
\end{equation}
applied coordinate-wise via $\hat{Q}_b^{(d)}(\mathbf{u})_k = \hat{q}_b^{(d)}(u_k)$ for any $\mathbf{u} \in \mathbb{R}^d$. The codebook uses no scales or zero-points and is shared by all layers and rows of the same input dimension $d$. Dequantization looks up centroids, rotates back by $\boldsymbol{\Phi}_d^\top$, and rescales by the stored norm.

\vspace{-1mm}
\section{Methodology}
\label{sec:method}
\vspace{-2mm}
\subsection{Overview}
\label{sec:method:overview}
\vspace{-1mm}
OrbitQuant replaces per-input range calibration with a distributional quantizer applied in one shared, rotated, normalized basis. Because weights and activations are quantized in the same basis, the rotation cancels in the matrix product and only a forward rotation on the activation remains at runtime. We quantize weights offline (Section~\ref{sec:method:weight}) and activations online (Section~\ref{sec:method:act}), realizing $\Pi_d$ as a randomized permuted block-Hadamard (RPBH) transform with a uniform random permutation (Section~\ref{sec:method:rpbh}). Figure~\ref{fig:overview} gives an overview.

\vspace{-1mm}
\subsection{Offline Weight Quantization}
\label{sec:method:weight}
For a linear layer with input dimension $d$, OrbitQuant uses the shared rotation $\boldsymbol{\Pi}_d$ of that dimension. Before inference we rotate the weight matrix into this basis,

\begin{equation}
\mathbf{W}' = \mathbf{W}\boldsymbol{\Pi}_d^\top.
\label{eq:weight-rotate}
\end{equation}
We split each row of $\mathbf{W}'$ into a magnitude $r'_i$ and a unit direction $\tilde{\mathbf{w}}'_i$,
\begin{equation}
r'_i = \|\mathbf{w}'_i\|_2, \quad \tilde{\mathbf{w}}'_i = \mathbf{w}'_i / r'_i, \quad i = 1, \ldots, m.
\label{eq:weight-norm}
\end{equation}
We then quantize the direction with the Lloyd--Max codebook of Section~\ref{sec:prelim:turboquant} and re-attach the magnitude,
\begin{equation}
\hat{\mathbf{W}}' = \mathrm{diag}(\mathbf{r}') \cdot \hat{Q}_{b_w}^{(d)}(\tilde{\mathbf{W}}').
\label{eq:weight-quant}
\end{equation}
Because $\boldsymbol{\Pi}_d$ is sampled independently of $\mathbf{w}_i$, each unit direction $\tilde{\mathbf{w}}'_i$ has coordinates following the density $f_d$ of Equation~\eqref{eq:beta-marginal}, so the Lloyd--Max codebook designed for $f_d$ is MSE-optimal on it. The row-norm vector $\mathbf{r}' \in \mathbb{R}^m$ is stored in BF16, adding $16m$ bits per layer, negligible against the $b_w m d$ bits of the quantized direction ($<0.3\%$). The original weight is replaced by $\hat{\mathbf{W}}'$ in place, so the inference path operates entirely in the rotated basis.

\begin{algorithm}[t]
\caption{OrbitQuant offline weight patching and online activation quantization}
\label{alg:OrbitQuant}
\small
\begin{algorithmic}[1]
\REQUIRE Transformer $\mathcal{T}$, target layers $\mathcal{L}$, input dimensions $\mathcal{D}$, bit-widths $(b_w, b_a)$, clamp $\varepsilon$
\STATE \textbf{$\triangleright$ Offline}
\FOR{$d \in \mathcal{D}$}
  \STATE $\boldsymbol{\Pi}_d \gets \mathrm{RPBH}(d)$
  \STATE $\hat{Q}_{b_w}^{(d)}, \hat{Q}_{b_a}^{(d)} \gets \textsc{LloydMax}(d, b_w), \textsc{LloydMax}(d, b_a)$
\ENDFOR
\FOR{each $\mathbf{W} \in \mathcal{L}$ with input dim $d$}
  \STATE $\mathbf{W}' \gets \mathbf{W}\boldsymbol{\Pi}_d^\top$
  \STATE $r'_i \gets \|\mathbf{w}'_i\|_2, \quad \tilde{\mathbf{w}}'_i \gets \mathbf{w}'_i / r'_i \quad \text{for } i = 1, \dots, m$
  \STATE $\hat{\mathbf{W}}' \gets \mathrm{diag}(\mathbf{r}')\,\hat{Q}_{b_w}^{(d)}(\tilde{\mathbf{W}}')$
  \STATE Replace $\mathbf{W}$ by $\hat{\mathbf{W}}'$ in $\mathcal{T}$
\ENDFOR
\medskip
\STATE \textbf{$\triangleright$ Online} on tokens $\mathbf{x} \in \mathbb{R}^{N \times d}$
\STATE $\mathbf{x}' \gets \mathbf{x}\boldsymbol{\Pi}_d^\top$ 
\STATE $s \gets \|\mathbf{x}'\|_2, \quad \tilde{\mathbf{x}}' \gets \mathbf{x}' / (s + \varepsilon)$
\STATE $\hat{\mathbf{x}}' \gets s \cdot \hat{Q}_{b_a}^{(d)}(\tilde{\mathbf{x}}')$
\RETURN $\hat{\mathbf{x}}'$
\end{algorithmic}
\end{algorithm}

\subsection{Online Activation Quantization}
\label{sec:method:act}
At inference, each incoming activation $\mathbf{x}$ is rotated by $\boldsymbol{\Pi}_d$ before it enters the layer and split into a magnitude $s$ and a unit direction $\tilde{\mathbf{x}}'$,
\begin{equation}
\mathbf{x}' = \boldsymbol{\Pi}_d \mathbf{x}, \quad s = \|\mathbf{x}'\|_2, \quad \tilde{\mathbf{x}}' = \mathbf{x}' / (s + \varepsilon),
\label{eq:act-norm}
\end{equation}
where $\varepsilon = 10^{-10}$ guards against zero norms on padding tokens. For a batch of $N$ tokens, this forward rotation $\boldsymbol{\Pi}_d\mathbf{x}$ is applied row-wise as $\mathbf{x}\boldsymbol{\Pi}_d^\top$. We quantize the direction with the Lloyd--Max quantizer $\hat{Q}_{b_a}^{(d)}$ and rescale by $s$,
\begin{equation}
\hat{\mathbf{x}}' = s \cdot \hat{Q}_{b_a}^{(d)}(\tilde{\mathbf{x}}').
\label{eq:act-quant}
\end{equation}
 As with the weights, $\tilde{\mathbf{x}}'$ has coordinates following $f_d$, so this codebook family applies without re-fitting. The only input-dependent quantity at inference is the per-token scalar $s$, while the codebook is fixed and calibration-free. Algorithm~\ref{alg:OrbitQuant} collects the offline and online stages. The weight absorbs $\boldsymbol{\Pi}_d^\top$ and the activation applies $\boldsymbol{\Pi}_d$, so the two cancel in the product, $\mathbf{W}'\mathbf{x}' = \mathbf{W}\boldsymbol{\Pi}_d^\top\boldsymbol{\Pi}_d\mathbf{x} = \mathbf{W}\mathbf{x}$. The quantized layer therefore computes $\hat{\mathbf{W}}'\hat{\mathbf{x}}' \approx \mathbf{W}\mathbf{x}$ with no inverse rotation at runtime.

\begin{figure*}[t]
\centering
\includegraphics[width=\linewidth]{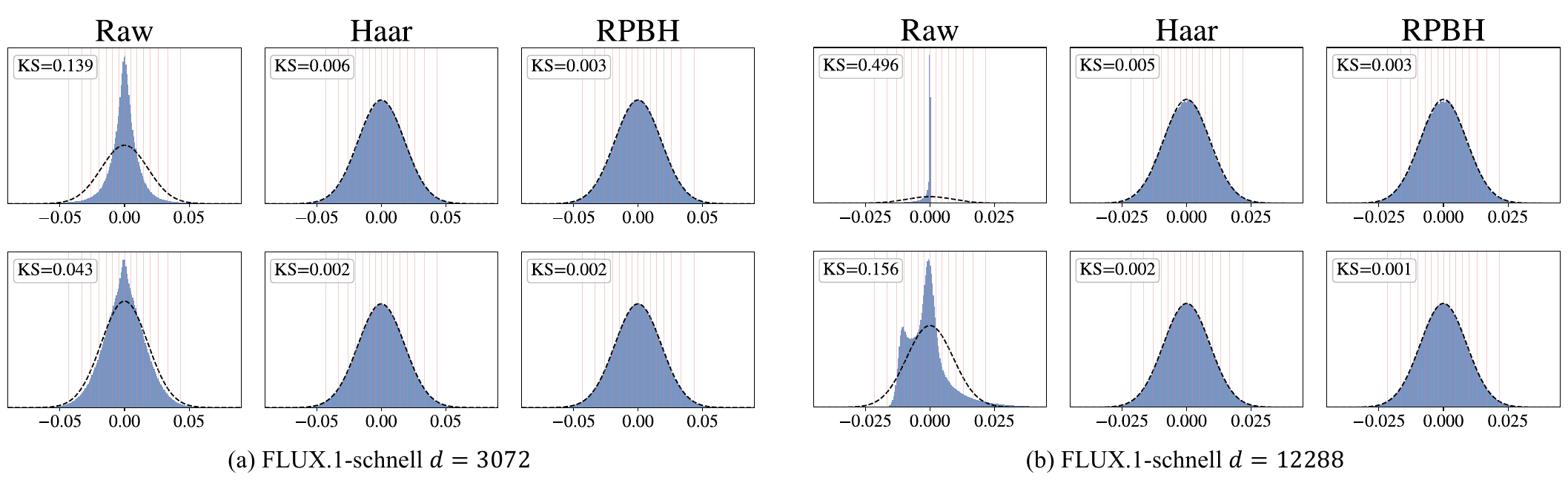}
\vspace{-8mm}
\caption{Rotated activation coordinates follow the dimension marginal $f_d$. For (a) an attention projection ($d{=}3072$) and (b) a feed-forward projection ($d{=}12288$) of FLUX.1-schnell, each cell plots the distribution of activation tokens with no rotation (Raw), a dense Haar rotation, and the RPBH. The dashed curve is the target $\mathcal{N}(0,1/d)$ and the inset reports the Kolmogorov--Smirnov distance to it. The light red vertical ticks mark the bin edges of the shared Lloyd--Max $W4$ codebook, which is fit to $f_d$ and reused for both weights and activations.}
\vspace{-4mm}
\label{fig:distribution}
\end{figure*}

\subsection{Randomized permuted block-Hadamard}
\label{sec:method:rpbh}
Quantizing all layers of dimension $d$ with one codebook built from $f_d$ works only if the rotated coordinates follow that marginal. A Haar rotation $\boldsymbol{\Phi}_d$ from~\cite{zandieh2025turboquant} makes them follow $f_d$ exactly. Since the rotation cancels in the matrix product for any orthogonal $\boldsymbol{\Pi}_d$, we are free to choose it for efficiency, as long as it keeps the marginal close to $f_d$. A dense Haar rotation costs $O(d^2)$ in both time per token and storage, which dominates the per-image activation cost. We instead realize $\boldsymbol{\Pi}_d$ as a randomized permuted block-Hadamard (RPBH) rotation~\cite{ailon2009fast, tropp2011improved},
\begin{equation}
\boldsymbol{\Pi}_d = \mathrm{blkdiag}(\mathbf{H}_h \mathbf{D}_1, \ldots, \mathbf{H}_h \mathbf{D}_{d/h}) \cdot \mathbf{P}_\pi,
\label{eq:rpbh}
\end{equation}
where $\mathrm{blkdiag}(\cdot)$ places its arguments as diagonal blocks, $\mathbf{H}_h$ is a $h \times h$ Walsh--Hadamard matrix, each $\mathbf{D}_i$ is a Rademacher sign diagonal, and $\mathbf{P}_\pi$ is the matrix of a uniform random permutation $\pi$ drawn once per dimension. It admits an $O(d \log h)$ transform through a permutation gather and a per-block Fast Walsh--Hadamard Transform, and stores as a sign vector and a permutation array rather than a $d \times d$ matrix. Unlike a full Randomized Hadamard Transform (RHT)~\cite{ashkboos2024quarot, tseng2024quip}, whose Walsh--Hadamard matrix exists only on power-of-two dimensions, the block form is constructible on any $d$. In practice, $h$ is the largest power of two dividing $d$, giving $h \in \{128, 512, 1024, 2048, 4096\}$ across all evaluated models.

The leading permutation $\mathbf{P}_\pi$ acts first and keeps the marginal close to $f_d$ at low bit-width. Without it, each block-Hadamard mixes only within its block, and an outlier concentrated in one block never spreads across the others. $\mathbf{P}_\pi$ spreads coordinates across blocks, so every block receives a balanced share of the input mass with high probability over $\pi$. Crucially, this permutation need not be data-dependent. Prior quantizers calibrate it by outlier magnitude~\cite{lin2024duquant}, column importance~\cite{gu2026lopro}, or per-block mass~\cite{sanjeet2026pushing}. RPBH instead draws it uniformly at random, which suffices for any input as the following proposition shows.

\begin{proposition}[Universal variance concentration]
\label{prop:marginal}
Let $\boldsymbol{\Pi}_d$ be the RPBH rotation of Equation~\eqref{eq:rpbh} on $d = kh$ with $k$ blocks of size $h$, and $\tilde{\mathbf{x}}$ a fixed unit vector with $\mu_\infty = \|\tilde{\mathbf{x}}\|_\infty^2$. For every $\delta \in (0, 1)$, with probability at least $1 - \delta$ over $\boldsymbol{\Pi}_d$, every coordinate $z_i$ of $\boldsymbol{\Pi}_d \tilde{\mathbf{x}}$ is centered with
\begin{equation}
\mathrm{Var}(z_i \mid \pi) \in \Big[\tfrac{1-\rho}{d},\; \tfrac{1+\rho}{d}\Big], \qquad \rho = d\,\mu_\infty \sqrt{\tfrac{1}{2h}\log\tfrac{4k}{\delta}}.
\label{eq:prop-var}
\end{equation}
\end{proposition}

Since $\rho$ stays small unless one coordinate carries an outsized share of the norm, the variance bound of Equation~\eqref{eq:prop-var} keeps the marginal of $\boldsymbol{\Pi}_d \tilde{\mathbf{x}}$ close to $\mathcal{N}(0, 1/d)$ and the Lloyd--Max codebook near-optimal. We prove the proposition in the supplementary material~\ref{sec:supp-proof-srht}. Section~\ref{sec:ablation:rotation} confirms that removing the permutation degrades low-bit robustness.

\subsection{Data-agnostic Codebook}
\label{sec:method:codebook}
Prior PTQ methods recalibrate because the activation range shifts with the timestep and prompt. OrbitQuant removes this dependence at the source. By Proposition~\ref{prop:marginal}, every coordinate of a normalized, RPBH-rotated activation stays close to the same marginal $f_d$, fixed by the dimension $d$ alone. We therefore run Lloyd--Max on $f_d$ offline to obtain a single codebook $\mathcal{C}_d$ per dimension, and quantizing reduces to normalizing, rotating, and mapping each coordinate to its nearest centroid, with no input statistics collected. One $\mathcal{C}_d$ serves every timestep, prompt, layer, and the weight rows of dimension $d$, which is what makes OrbitQuant calibration-free.

This codebook follows TurboQuant~\cite{zandieh2025turboquant}, which quantizes randomly rotated vectors against a fixed Beta-marginal codebook computed once. TurboQuant is a standalone vector quantizer for KV-cache and vector-database compression. It uses a dense $O(d^2)$ Haar rotation and operates as a quantize-dequantize codec, reconstructing each vector before use. OrbitQuant instead pairs the codebook with the structured RPBH rotation and absorbs it into the weights. The quantized operands then feed each linear layer directly, with no reconstruction and only a forward rotation at inference.
Figure~\ref{fig:distribution} confirms the marginal matching at an attention projection ($d{=}3072$) and a feed-forward layer ($d{=}12288$). Raw activations deviate sharply from $f_d$, but after the RPBH rotation both weights and activations match $f_d \approx \mathcal{N}(0,1/d)$ as closely as a dense Haar rotation does, so a single codebook built from $f_d$ fits them all.

\section{Experiments}
\label{sec:exp}

\subsection{Setup}
\label{sec:exp:setup}

\begin{table*}[t]
\caption{GenEval results on three image diffusion transformers at W4A4 and W2A4. Values are scores on six compositional sub-tasks and Overall. \textbf{Bold} and \underline{underlined} entries indicate the best and second-best result within each (model, bit-width) group. $\uparrow$ means higher is better. $\dagger$ represents our implementation.}
\vspace{-2mm}
\label{tab:geneval}
\small
\centering
\resizebox{0.9\textwidth}{!}{%
\begin{tabular}{clcccccccc}
\toprule
Model & Method & Bit & Single object~$\uparrow$ & Two object~$\uparrow$ & Counting~$\uparrow$ & Colors~$\uparrow$ & Position~$\uparrow$ & Color attribution~$\uparrow$ & Overall~$\uparrow$ \\
\midrule
\multirow{12}{*}{FLUX.1-schnell}
  & FP16                                  & 16/16 & 0.997 & 0.884 & 0.600 & 0.742 & 0.275 & 0.488 & 0.664 \\
  \cdashline{2-10}\\[-1.75ex]
  & Q-DiT~\cite{chen2025q}                  & W4A4 & 0.741 & 0.424 & 0.378 & 0.418 & 0.073 & 0.208 & 0.373 \\
  & SmoothQuant~\cite{xiao2023smoothquant} & W4A4 & 0.619 & 0.293 & 0.272 & 0.317 & 0.043 & 0.143 & 0.281 \\
  & QuaRot~\cite{ashkboos2024quarot}      & W4A4 & 0.819 & 0.543 & 0.472 & 0.519 & 0.118 & 0.275 & 0.458 \\
  & ViDiT-Q~\cite{zhao2024vidit}          & W4A4 & 0.888 & 0.586 & 0.516 & 0.585 & 0.130 & 0.268 & 0.495 \\
  & SVDQuant~\cite{li2024svdquant}        & W4A4 & \underline{0.994} & \textbf{0.910} & 0.450 & 0.708 & 0.260 & 0.420 & 0.624 \\
  & AdaTSQ~\cite{zhang2026adatsq}         & W4A4 & \textbf{0.997} & \underline{0.894} & \underline{0.622} & \underline{0.793} & \underline{0.278} & \underline{0.498} & \underline{0.680} \\
  & \cellcolor{gray!25}OrbitQuant            & \cellcolor{gray!25}W4A4 & \cellcolor{gray!25}0.991 & \cellcolor{gray!25}0.881 & \cellcolor{gray!25}\textbf{0.706} & \cellcolor{gray!25}\textbf{0.803} & \cellcolor{gray!25}\textbf{0.323} & \cellcolor{gray!25}\textbf{0.512} & \cellcolor{gray!25}\textbf{0.703} \\
  \cdashline{2-10}\\[-1.75ex]
  & QuaRot$\dagger$~\cite{ashkboos2024quarot}      & W2A4 & 0.006 & 0.000 & 0.000 & 0.000 & 0.000 & 0.000 & 0.001 \\
  & SmoothQuant$\dagger$~\cite{xiao2023smoothquant} & W2A4 & 0.000 & 0.000 & 0.000 & 0.000 & 0.000 & 0.000 & 0.000 \\
  & ViDiT-Q$\dagger$~\cite{zhao2024vidit}          & W2A4 & 0.006 & 0.000 & 0.000 & 0.000 & 0.000 & 0.000 & 0.001 \\
  & \cellcolor{gray!25}OrbitQuant            & \cellcolor{gray!25}W2A4 & \cellcolor{gray!25}\textbf{0.972} & \cellcolor{gray!25}\textbf{0.697} & \cellcolor{gray!25}\textbf{0.575} & \cellcolor{gray!25}\textbf{0.766} & \cellcolor{gray!25}\textbf{0.198} & \cellcolor{gray!25}\textbf{0.420} & \cellcolor{gray!25}\textbf{0.604} \\
\midrule
\multirow{12}{*}{FLUX.1-dev}
  & FP16                                  & 16/16 & 0.984 & 0.823 & 0.769 & 0.771 & 0.203 & 0.450 & 0.667 \\
  \cdashline{2-10}\\[-1.75ex]
  & Q-DiT~\cite{chen2025q}                  & W4A4 & 0.047 & 0.000 & 0.009 & 0.024 & 0.000 & 0.003 & 0.014 \\
  & SmoothQuant~\cite{xiao2023smoothquant} & W4A4 & 0.003 & 0.000 & 0.003 & 0.011 & 0.000 & 0.000 & 0.007 \\
  & QuaRot~\cite{ashkboos2024quarot}      & W4A4 & 0.634 & 0.106 & 0.294 & 0.346 & 0.025 & 0.050 & 0.243 \\
  & ViDiT-Q~\cite{zhao2024vidit}          & W4A4 & 0.709 & 0.147 & 0.325 & 0.410 & 0.028 & 0.060 & 0.280 \\
  & SVDQuant~\cite{li2024svdquant}        & W4A4 & \underline{0.981} & 0.710 & 0.610 & 0.698 & 0.140 & 0.300 & 0.573 \\
  & AdaTSQ~\cite{zhang2026adatsq}         & W4A4 & \underline{0.981} & \textbf{0.770} & \underline{0.640} & \underline{0.708} & \textbf{0.260} & \underline{0.350} & \underline{0.618} \\
  & \cellcolor{gray!25}OrbitQuant            & \cellcolor{gray!25}W4A4 & \cellcolor{gray!25}\textbf{0.988} & \cellcolor{gray!25}\underline{0.768} & \cellcolor{gray!25}\textbf{0.691} & \cellcolor{gray!25}\textbf{0.755} & \cellcolor{gray!25}\underline{0.178} & \cellcolor{gray!25}\textbf{0.420} & \cellcolor{gray!25}\textbf{0.633} \\
  \cdashline{2-10}\\[-1.75ex]
  & QuaRot$\dagger$~\cite{ashkboos2024quarot}      & W2A4 &0.006 & 0.000 & 0.000 & 0.000 & 0.000 & 0.000 & 0.001 \\
  & SmoothQuant$\dagger$~\cite{xiao2023smoothquant} & W2A4 & 0.000 & 0.000 & 0.000 & 0.000 & 0.000 & 0.000 & 0.000 \\
  & ViDiT-Q$\dagger$~\cite{zhao2024vidit}          & W2A4 & 0.006 & 0.000 & 0.000 & 0.000 & 0.000 & 0.000 & 0.001 \\
& \cellcolor{gray!25}OrbitQuant            & \cellcolor{gray!25}W2A4 & \cellcolor{gray!25}\textbf{0.956} & \cellcolor{gray!25}\textbf{0.424} & \cellcolor{gray!25}\textbf{0.481} & \cellcolor{gray!25}\textbf{0.678} & \cellcolor{gray!25}\textbf{0.110} & \cellcolor{gray!25}\textbf{0.203} & \cellcolor{gray!25}\textbf{0.475} \\
\midrule
\multirow{11}{*}{Z-Image-Turbo}
  & FP16                                  & 16/16 & 1.000 & 0.907 & 0.709 & 0.859 & 0.468 & 0.583 & 0.754 \\
  \cdashline{2-10}\\[-1.75ex]
  & SmoothQuant~\cite{xiao2023smoothquant} & W4A4 & 0.003 & 0.000 & 0.000 & 0.000 & 0.000 & 0.000 & 0.000 \\
  & QuaRot~\cite{ashkboos2024quarot}      & W4A4 & 0.906 & 0.505 & 0.416 & 0.692 & 0.250 & 0.343 & 0.519 \\
  & ViDiT-Q~\cite{zhao2024vidit}          & W4A4 & 0.972 & 0.705 & 0.584 & 0.777 & 0.435 & 0.533 & 0.668 \\
  & SVDQuant~\cite{li2024svdquant}        & W4A4 & \underline{0.994} & 0.843 & 0.633 & 0.833 & \underline{0.485} & 0.520 & 0.718 \\
  & AdaTSQ~\cite{zhang2026adatsq}         & W4A4 & \underline{0.994} & \textbf{0.891} & \underline{0.681} & \underline{0.872} & \textbf{0.520} & \textbf{0.613} & \underline{0.762} \\
  & \cellcolor{gray!25}OrbitQuant            & \cellcolor{gray!25}W4A4 & \cellcolor{gray!25}\textbf{0.997} & \cellcolor{gray!25}\underline{0.889} & \cellcolor{gray!25}\textbf{0.781} & \cellcolor{gray!25}\textbf{0.888} & \cellcolor{gray!25}0.450 & \cellcolor{gray!25}\underline{0.598} & \cellcolor{gray!25}\textbf{0.767} \\
  \cdashline{2-10}\\[-1.75ex]
  & QuaRot$\dagger$~\cite{ashkboos2024quarot}      & W2A4 & 0.006 & 0.000 & 0.000 & 0.000 & 0.000 & 0.000 & 0.001 \\
  & SmoothQuant$\dagger$~\cite{xiao2023smoothquant} & W2A4 & 0.006 & 0.000 & 0.000 & 0.000 & 0.000 & 0.000 & 0.001 \\
  & ViDiT-Q$\dagger$~\cite{zhao2024vidit}          & W2A4 & 0.003 & 0.000 & 0.000 & 0.003 & 0.000 & 0.000 & 0.001 \\
  & \cellcolor{gray!25}OrbitQuant            & \cellcolor{gray!25}W2A4 & \cellcolor{gray!25}\textbf{0.703} & \cellcolor{gray!25}\textbf{0.194} & \cellcolor{gray!25}\textbf{0.275} & \cellcolor{gray!25}\textbf{0.500} & \cellcolor{gray!25}\textbf{0.128} & \cellcolor{gray!25}\textbf{0.113} & \cellcolor{gray!25}\textbf{0.319} \\
\bottomrule
\end{tabular}}
\vspace{-5mm}
\end{table*}

\paragraph{Models and bit-widths.}
We evaluate OrbitQuant on three image DiTs and two video DiTs. For image generation we report FLUX.1-schnell (4-step, guidance 0.0), FLUX.1-dev (50-step, guidance 3.5), and Z-Image-Turbo (10-step, guidance 0.0) at W4A4 and W2A4. For video generation we report Wan~2.1-1.3B (81 frames, $480{\times}832$, 50 steps, CFG 5.0) and CogVideoX-2B (49 frames, $480{\times}720$, 50 steps, CFG 6.0) at W4A6 and W4A4. We quantize all transformer-block projections with OrbitQuant and keep the adaptive layer normalization (AdaLN) modulation projections, where present, at INT4 weight round-to-nearest (RTN)~\cite{li2024svdquant}. This AdaLN treatment is identical across all methods we implement. Wan~2.1-1.3B has no AdaLN modulation, so only its transformer-block projections are quantized. The full list of quantized and skipped layers is given in the supplementary material~\ref{sec:supp-additional}.


\paragraph{Baselines.}
Image baselines are the calibration-based SVDQuant~\cite{li2024svdquant}, AdaTSQ~\cite{zhang2026adatsq}, and ViDiT-Q~\cite{zhao2024vidit}, together with Q-DiT~\cite{chen2025q}, QuaRot~\cite{ashkboos2024quarot}, and SmoothQuant~\cite{xiao2023smoothquant}. Video baselines are ViDiT-Q, SVDQuant, QuaRot, and SmoothQuant. Baseline numbers are taken primarily from AdaTSQ~\cite{zhang2026adatsq} for image and QVGen~\cite{huang2025qvgen} for video.

\subsection{Image generation: GenEval}
\label{sec:exp:image}
Table~\ref{tab:geneval} reports GenEval Overall and per-task scores. At W4A4, OrbitQuant is essentially lossless, exceeding FP16 on Overall on FLUX.1-schnell and Z-Image-Turbo and trailing it by 0.034 on FLUX.1-dev, while outperforming every PTQ baseline to set the state of the art on GenEval. The advantage widens at W2A4, where the rotation and smoothing baselines collapse to near-zero on all three backbones while OrbitQuant stays functional, retaining most of its quality on the FLUX models and remaining the only method that produces meaningful scores on Z-Image-Turbo. Results at W3A3 and W2A3 are presented in the supplementary material~\ref{supple:low_bit}.

\begin{table*}[t]
\caption{VBench PTQ video-generation results on Wan~2.1-1.3B and CogVideoX-2B. Scores are percentages. \textbf{Bold} and \underline{underlined} entries indicate the best and second-best result within each (model, bit-width) group. $\dagger$ represents our implementation.}
\vspace{-2mm}
\label{tab:vbench}
\centering
\resizebox{0.9\textwidth}{!}{%
\begin{tabular}{clccccccccc}
\toprule
Model & Method & Bit & \makecell{Imaging\\Quality~$\uparrow$} & \makecell{Aesthetic\\Quality~$\uparrow$} & \makecell{Motion\\Smoothness~$\uparrow$} & \makecell{Dynamic\\Degree~$\uparrow$} & \makecell{Background\\Consistency~$\uparrow$} & \makecell{Subject\\Consistency~$\uparrow$} & Scene~$\uparrow$ & \makecell{Overall\\Consistency~$\uparrow$} \\
\midrule
\multirow{11}{*}{Wan~2.1-1.3B}
  & Full Prec. & 16/16 & 64.30 & 58.21 & 97.37 & 70.28 & 95.94 & 93.84 & 28.05 & 24.67 \\
  \cdashline{2-11}\\[-1.75ex]
  & SmoothQuant$\dagger$~\cite{xiao2023smoothquant} & W4A6 & 53.51 & 49.19 & \textbf{98.01} & 34.44 & 94.89 & \underline{92.66} & 12.81 & 22.15 \\
  & QuaRot$\dagger$~\cite{ashkboos2024quarot}       & W4A6 & 56.92 & 50.36 & 96.94 & \underline{54.17} & \underline{95.36} & 91.65 & \underline{14.88} & 22.65 \\
  & ViDiT-Q~\cite{zhao2024vidit}           & W4A6 & 56.24 & 50.18 & 94.81 & 52.43 & 89.67 & 82.53 & 13.45 & 19.58 \\
  & SVDQuant~\cite{li2024svdquant}         & W4A6 & \underline{58.16} & \underline{51.27} & 97.05 & 49.44 & 93.74 & 91.71 & 14.18 & \underline{23.26} \\
  & \cellcolor{gray!25}OrbitQuant & \cellcolor{gray!25}W4A6 & \cellcolor{gray!25}\textbf{61.25} & \cellcolor{gray!25}\textbf{56.08} & \cellcolor{gray!25}\underline{97.76} & \cellcolor{gray!25}\textbf{59.78} & \cellcolor{gray!25}\textbf{95.51} & \cellcolor{gray!25}\textbf{94.23} & \cellcolor{gray!25}\textbf{24.88} & \cellcolor{gray!25}\textbf{24.35} \\
  \cdashline{2-11}\\[-1.75ex]
  & SmoothQuant$\dagger$~\cite{xiao2023smoothquant} & W4A4 & 46.32 & 36.33 & \underline{96.39} & 51.94 & \underline{95.85} & \underline{90.39} & 2.79 & 15.05 \\
  & QuaRot$\dagger$~\cite{ashkboos2024quarot}       & W4A4 & 51.42 & 40.49 & 96.21 & 52.78 & 95.76 & 88.80 & 5.31 & 17.98 \\
  & ViDiT-Q$\dagger$~\cite{zhao2024vidit}           & W4A4 & 44.51 & 36.43 & 96.16 & \underline{58.06} & \textbf{95.92} & 89.59 & 1.85 & 13.11 \\
  & SVDQuant~\cite{li2024svdquant}         & W4A4 & \underline{57.57} & \underline{46.30} & 94.21 & \textbf{72.22} & 93.16 & 77.96 & \underline{12.73} & \underline{21.91} \\
  & \cellcolor{gray!25}OrbitQuant & \cellcolor{gray!25}W4A4 & \cellcolor{gray!25}\textbf{58.58} & \cellcolor{gray!25}\textbf{53.41} & \cellcolor{gray!25}\textbf{97.42} & \cellcolor{gray!25}53.89 & \cellcolor{gray!25}95.30 & \cellcolor{gray!25}\textbf{92.98} & \cellcolor{gray!25}\textbf{18.81} & \cellcolor{gray!25}\textbf{23.86} \\
\midrule
\multirow{11}{*}{CogVideoX-2B}
  & Full Prec. & 16/16 & 59.15 & 54.49 & 97.43 & 67.78 & 94.79 & 92.82 & 36.24 & 25.06 \\
  \cdashline{2-11}\\[-1.75ex]
  & SmoothQuant$\dagger$~\cite{xiao2023smoothquant} & W4A6 & 51.50 & 49.70 & \textbf{97.20} & 30.00 & \underline{94.70} & 91.10 & 21.90 & 23.20 \\
  & QuaRot$\dagger$~\cite{ashkboos2024quarot}       & W4A6 & 54.11 & \underline{52.25} & 96.92 & \underline{49.72} & 94.60 & \underline{91.82} & \underline{30.73} & \underline{24.03} \\
  & ViDiT-Q~\cite{zhao2024vidit}           & W4A6 & 54.72 & 43.01 & 92.18 & 43.22 & 90.76 & 81.02 & 26.25 & 20.41 \\
  & SVDQuant~\cite{li2024svdquant}         & W4A6 & \textbf{58.27} & 47.06 & 95.28 & 40.83 & 92.41 & 87.45 & 27.69 & 21.34 \\
  & \cellcolor{gray!25}OrbitQuant & \cellcolor{gray!25}W4A6 & \cellcolor{gray!25}\underline{55.59} & \cellcolor{gray!25}\textbf{54.42} & \cellcolor{gray!25}\underline{97.02} & \cellcolor{gray!25}\textbf{57.50} & \cellcolor{gray!25}\textbf{94.78} & \cellcolor{gray!25}\textbf{92.56} & \cellcolor{gray!25}\textbf{32.51} & \cellcolor{gray!25}\textbf{24.55} \\
  \cdashline{2-11}\\[-1.75ex]
  & SmoothQuant$\dagger$~\cite{xiao2023smoothquant} & W4A4 & 39.90 & 35.50 & \textbf{97.80} & 1.90 & \textbf{95.90} & \textbf{92.90} & 3.60 & 12.80 \\
  & QuaRot$\dagger$~\cite{ashkboos2024quarot}       & W4A4 & 49.60 & 47.10 & 96.90 & 9.20 & 94.80 & 90.20 & 19.70 & 21.70 \\
  & ViDiT-Q$\dagger~$\cite{zhao2024vidit}           & W4A4 & 44.80 & 42.10 & 97.30 & 4.40 & \underline{95.60} & 90.30 & 10.50 & 18.40 \\
  & SVDQuant~\cite{li2024svdquant}         & W4A4 & \underline{51.60} & \underline{49.40} & \underline{97.69} & \underline{42.22} & 94.03 & \underline{91.78} & \underline{25.67} & \underline{22.89} \\
  & \cellcolor{gray!25}OrbitQuant & \cellcolor{gray!25}W4A4 & \cellcolor{gray!25}\textbf{52.62} & \cellcolor{gray!25}\textbf{51.66} & \cellcolor{gray!25}96.99 & \cellcolor{gray!25}\textbf{42.78} & \cellcolor{gray!25}94.50 & \cellcolor{gray!25}91.65 & \cellcolor{gray!25}\textbf{28.53} & \cellcolor{gray!25}\textbf{23.86} \\
\bottomrule
\end{tabular}}
\vspace{-2mm}
\end{table*}

\begin{figure*}[t]
\centering
\includegraphics[width=\linewidth]{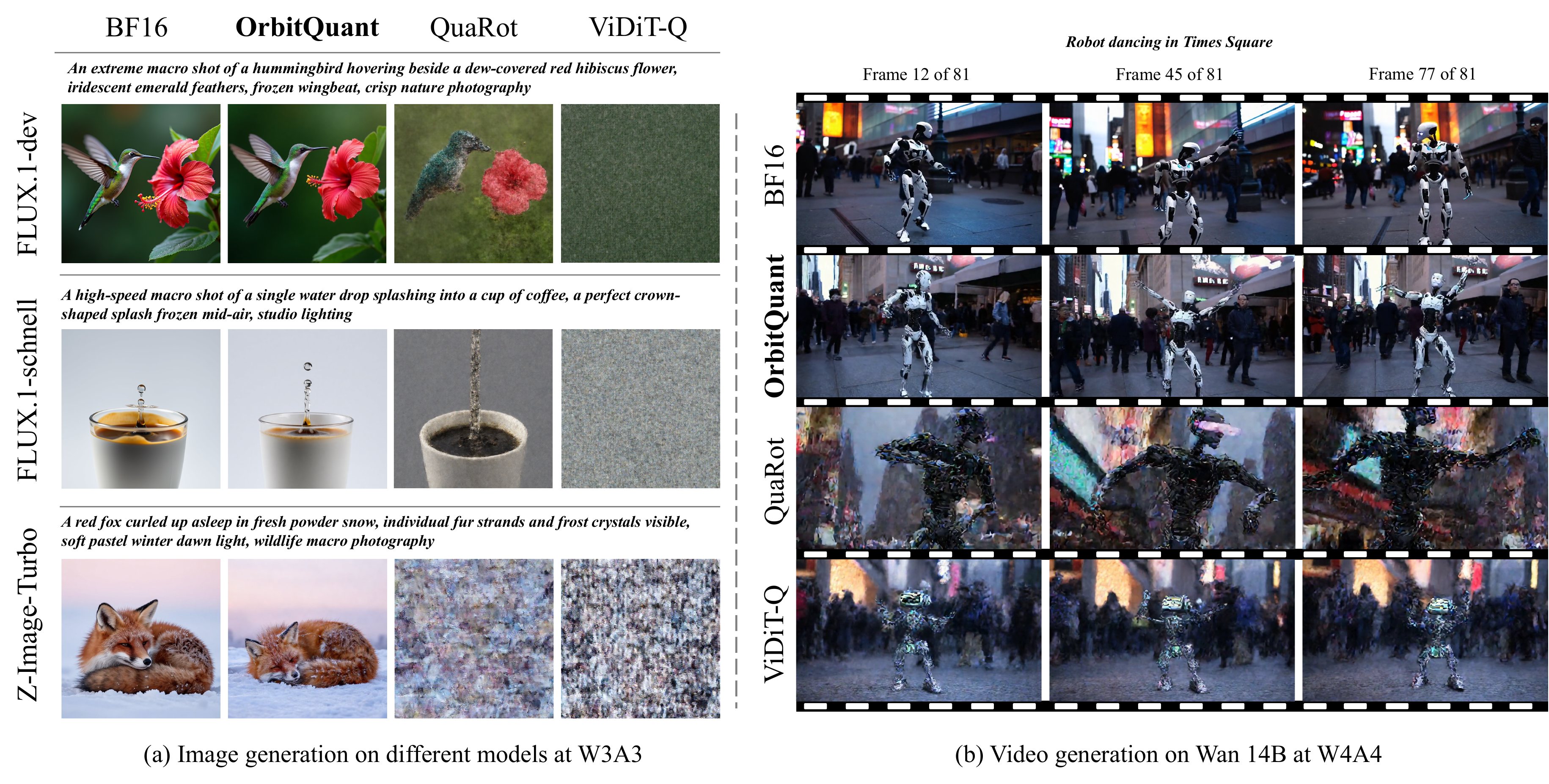}
\vspace{-7mm}
\caption{Qualitative comparison of OrbitQuant against QuaRot~\cite{ashkboos2024quarot} and ViDiT-Q~\cite{zhao2024vidit}, with the BF16 full-precision output shown for reference. \textbf{(a)} Image generation at W3A3 on FLUX.1-dev, FLUX.1-schnell, and Z-Image-Turbo, with one prompt per model. \textbf{(b)} Video generation at W4A4 on Wan 14B, showing three sampled frames per method.}
\label{fig:qualitative}
\vspace{-2.5mm}
\end{figure*}

\subsection{Video generation: VBench}
\label{sec:exp:video}
OrbitQuant applies to Wan~2.1-1.3B~\cite{wan2025wan} and CogVideoX-2B~\cite{yang2025cogvideox} with the identical recipe used for the image experiments, and Table~\ref{tab:vbench} reports the full VBench comparison. At W4A6, OrbitQuant is the strongest PTQ method on both backbones, leading on Overall Consistency and most per-dimension scores over the calibration-based ViDiT-Q~\cite{zhao2024vidit} and SVDQuant~\cite{li2024svdquant}. The advantage holds at W4A4, where the baselines lose ground on the harder dimensions while OrbitQuant stays closest on most dimensions, again ranking first on Overall Consistency on both backbones. Comparisons against quantization-aware training (QAT) and results on the huge model, including Wan~14B~\cite{wan2025wan} and HunyuanVideo~\cite{kong2024hunyuanvideo}, are given in the supplementary material~\ref{supple:add-exp}.



\subsection{Qualitative Comparison}
\label{sec:exp:qualitative}
Figure~\ref{fig:qualitative} shows generations from OrbitQuant, QuaRot~\cite{ashkboos2024quarot}, and ViDiT-Q~\cite{zhao2024vidit} alongside the BF16 reference. For images at W3A3, OrbitQuant stays close to BF16 on FLUX.1-dev, FLUX.1-schnell, and Z-Image-Turbo, retaining fine structure and color, while the other methods lose fidelity and collapse to noise on Z-Image-Turbo. For Wan~14B video at W4A4, OrbitQuant preserves scene layout and stays consistent across frames, whereas the other methods drift in color and structure.

\subsection{Latency and Memory Analysis}
We measure end-to-end latency and peak memory on FLUX.1-dev for image (NVIDIA H100, $1024^2$, $50$ steps, guidance $3.5$) and on Wan~2.1-1.3B for video ($480{\times}832$, $81$ frames, $50$ steps, CFG $5.0$). All methods are evaluated under fake quantization, with weights and activations dequantized to BF16 and the matmul computed in BF16. The comparison therefore measures quantization overhead, not realized low-bit speedup. As shown in Figure~\ref{fig:efficiency}, OrbitQuant has the lowest overhead among the weight-and-activation quantization methods on both, with SmoothQuant~\cite{xiao2023smoothquant}, QuaRot~\cite{ashkboos2024quarot}, and ViDiT-Q~\cite{zhao2024vidit} running $1.09\times$, $1.28\times$, and $1.40\times$ slower on image and in the same relative order on video. OrbitQuant keeps the lowest peak memory on image, matching the unquantized footprint.

OrbitQuant has the lowest overhead because its activation quantization is a fixed, shared-codebook nearest-centroid lookup, which empirically undercuts the dynamic per-token uniform quantization of QuaRot and the additional channel-wise smoothing of SmoothQuant and ViDiT-Q. On video, where activations dominate, the lookup materializes an index and gather tensor that lifts OrbitQuant's peak memory above QuaRot and SmoothQuant (20.3 vs 19.3~GB), though still below ViDiT-Q (23.2~GB).

\begin{figure}[t]
\centering
\includegraphics[width=\linewidth]{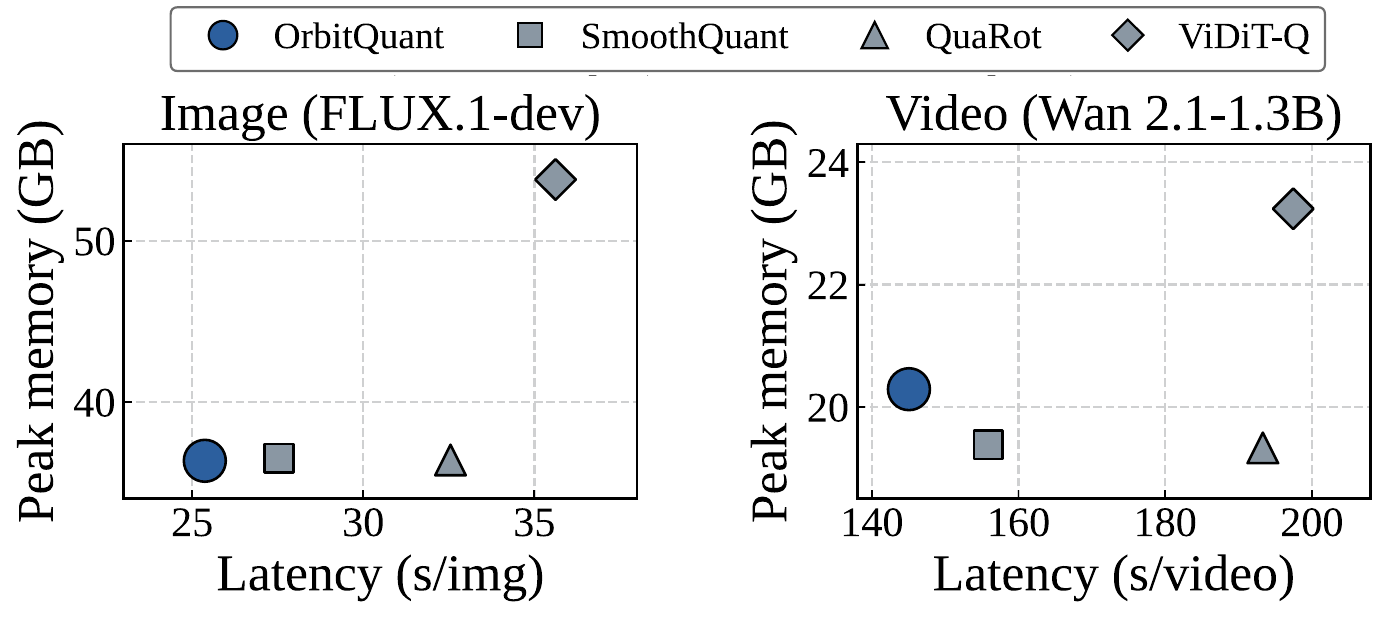}
\caption{Latency and peak memory together, with lower-left being better. The left panel is image generation on FLUX.1-dev, the right panel video on Wan~2.1-1.3B.}

\label{fig:efficiency}
\end{figure}

\section{Ablations}
\label{sec:ablation}

\begin{table}[t]
\centering
\caption{Rotation-class ablation on FLUX.1-schnell. GenEval Overall (mean over three seeds) at three bit-widths, and the per-image activation-rotation latency at $1024^2$ on an H100 (rotation cost only, summed over all quantized layers and denoising steps).}
\label{tab:rotation}
\resizebox{0.85\columnwidth}{!}{%
\begin{tabular}{lcccc}
\toprule
Rotation & W4A4 & W3A3 & W2A4 & Latency(s) \\
\midrule
Haar  & \textbf{0.696} & 0.669 & 0.591 & 11.65 \\
Full RHT  & 0.691 & 0.672 & 0.587 & 0.452 \\
Block-RHT    & 0.678 & 0.642 & 0.558 & 0.381 \\
\rowcolor{gray!15}
\textbf{RPBH (ours)}                & 0.690 & \textbf{0.674} & \textbf{0.595}  & \textbf{0.451} \\
\bottomrule
\end{tabular}}
\vspace{-2mm}
\end{table}

\subsection{Comparison between Rotation Matrix}
\label{sec:ablation:rotation}
The forward identity of Section~\ref{sec:method:act} holds for any orthogonal rotation, but codebook compatibility requires the rotated marginal to match $f_d$. We compare four rotations inside an otherwise identical OrbitQuant pipeline on FLUX.1-schnell. Table~\ref{tab:rotation} reports GenEval Overall at three bit-widths and the per-image activation rotation latency on an H100.
At W4A4 the four rotations are within noise. They separate at lower bit-widths, where RPBH is the strongest at W3A3 and W2A4, ahead of the dense Haar, the permutation-free Block-Randomized Hadamard Transform (Block-RHT), and the Full RHT of the kind used by rotation-based LLM quantizers~\cite{ashkboos2024quarot, tseng2024quip}. The random permutation drives the gap over Block-RHT, which is RPBH without the permutation. It spreads clustered outliers across blocks so the rotated marginal stays close to $f_d$, which a fixed codebook can quantize at low bit-width. The structured rotations admit a fast Hadamard transform kernel that the dense Haar cannot use, running an order of magnitude faster ($26\times$). Among them RPBH adds 0.070\,s over Block-RHT and is no slower than the Full RHT, while remaining constructible on every dimension in our study, including $d{=}1920$ of CogVideoX-2B where no fast size-$d$ Hadamard kernel exists.

\subsection{AdaLN bit-width}
\label{sec:ablation:adaln}
OrbitQuant fixes AdaLN modulation projections at INT4 weight RTN regardless of the main bit-width, since their timestep-dependent scale-and-shift cannot be folded into neighboring weights. To isolate this choice, we hold the main model at W4A4 and vary only the AdaLN weight bit, keeping AdaLN activations in BF16. Figure~\ref{fig:adaln} reports GenEval Overall on three models.
Quantizing the AdaLN weights to INT4 nearly matches the BF16 result on all three models, and lowering them further degrades Overall in a model-dependent way. At W3 all three models hold, but at W2 FLUX.1-dev and -schnell collapse, while Z-Image-Turbo stays robust. A low-bit AdaLN weight corrupts the modulation that every downstream layer reads. We still quantize these projections to INT4 rather than keep them in BF16, since they are $27\%$ of the weights and leaving them in BF16 would drop the FLUX model compression from $4\times$ to $2.21\times$ (Figure~\ref{fig:adaln}, right). Pushing them to W2 saves further memory while triggering this collapse on the FLUX models, so OrbitQuant keeps AdaLN at INT4.

\begin{figure}[t]
\centering
\includegraphics[width=\linewidth]{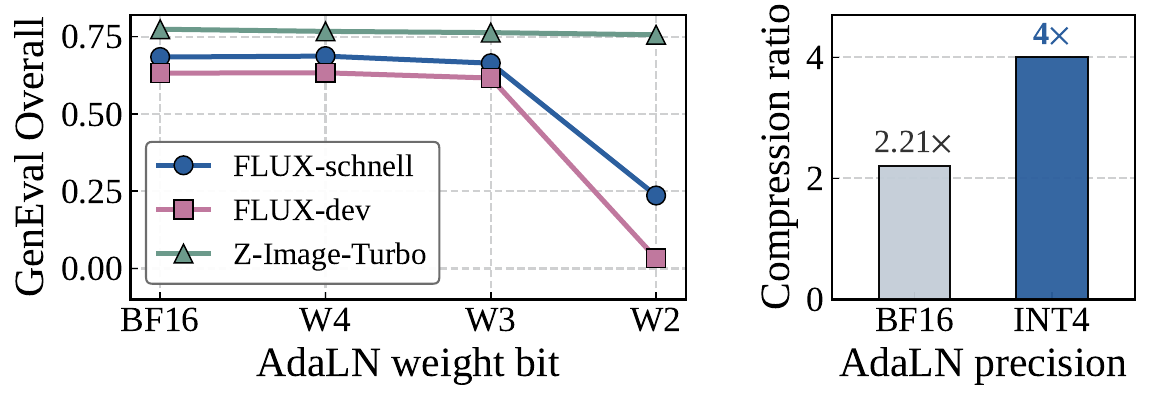}
\vspace{-3mm}
\caption{AdaLN bit-width ablation with the main model fixed at W4A4 and AdaLN activations in BF16. The left panel reports GenEval Overall as the AdaLN modulation weight bit drops. The right panel reports model compression on the FLUX architecture with the AdaLN weights in BF16 ($2.21\times$) and in INT4 ($4\times$).}
\label{fig:adaln}
\vspace{-2mm}
\end{figure}

\section{Conclusion}
We present OrbitQuant, a calibration-free weight-activation quantizer for diffusion transformers that replaces per-timestep range calibration with a single distributional codebook applied in a shared, rotated, normalized basis. The rotation is absorbed into the weights offline and cancels inside each linear layer, leaving only a forward RPBH rotation on the activations at runtime. Its random permutation is what keeps the rotated marginal well-behaved at low bit-width. Across FLUX.1, Z-Image-Turbo, Wan~2.1, and CogVideoX, the same recipe transfers from image to video with no per-modality tuning, sets the state of the art for PTQ on GenEval and VBench at low-bit settings, and supports usable 2-bit weights where prior PTQ methods collapse.

\clearpage
{
    \small
    \bibliographystyle{ieeenat_fullname}
    \bibliography{main}

@inproceedings{xiao2023smoothquant,
  title={Smoothquant: Accurate and efficient post-training quantization for large language models},
  author={Xiao, Guangxuan and Lin, Ji and Seznec, Mickael and Wu, Hao and Demouth, Julien and Han, Song},
  booktitle={International conference on machine learning},
  pages={38087--38099},
  year={2023},
  organization={PMLR}
}

@article{ashkboos2024quarot,
  title={Quarot: Outlier-free 4-bit inference in rotated llms},
  author={Ashkboos, Saleh and Mohtashami, Amirkeivan and Croci, Maximilian L and Li, Bo and Cameron, Pashmina and Jaggi, Martin and Alistarh, Dan and Hoefler, Torsten and Hensman, James},
  journal={Advances in Neural Information Processing Systems},
  volume={37},
  pages={100213--100240},
  year={2024}
}

@inproceedings{peebles2023scalable,
  title={Scalable diffusion models with transformers},
  author={Peebles, William and Xie, Saining},
  booktitle={Proceedings of the IEEE/CVF international conference on computer vision},
  pages={4195--4205},
  year={2023}
}

@inproceedings{esser2024scaling,
  title={Scaling rectified flow transformers for high-resolution image synthesis},
  author={Esser, Patrick and Kulal, Sumith and Blattmann, Andreas and Entezari, Rahim and M{\"u}ller, Jonas and Saini, Harry and Levi, Yam and Lorenz, Dominik and Sauer, Axel and Boesel, Frederic and others},
  booktitle={Forty-first international conference on machine learning},
  year={2024}
}

@article{singer2022make,
  title={Make-a-video: Text-to-video generation without text-video data},
  author={Singer, Uriel and Polyak, Adam and Hayes, Thomas and Yin, Xi and An, Jie and Zhang, Songyang and Hu, Qiyuan and Yang, Harry and Ashual, Oron and Gafni, Oran and others},
  journal={arXiv preprint arXiv:2209.14792},
  year={2022}
}

@inproceedings{bao2023all,
  title={All are worth words: A vit backbone for diffusion models},
  author={Bao, Fan and Nie, Shen and Xue, Kaiwen and Cao, Yue and Li, Chongxuan and Su, Hang and Zhu, Jun},
  booktitle={Proceedings of the IEEE/CVF conference on computer vision and pattern recognition},
  pages={22669--22679},
  year={2023}
}

@inproceedings{rombach2022high,
  title={High-resolution image synthesis with latent diffusion models},
  author={Rombach, Robin and Blattmann, Andreas and Lorenz, Dominik and Esser, Patrick and Ommer, Bj{\"o}rn},
  booktitle={Proceedings of the IEEE/CVF conference on computer vision and pattern recognition},
  pages={10684--10695},
  year={2022}
}

@article{saharia2022photorealistic,
  title={Photorealistic text-to-image diffusion models with deep language understanding},
  author={Saharia, Chitwan and Chan, William and Saxena, Saurabh and Li, Lala and Whang, Jay and Denton, Emily L and Ghasemipour, Kamyar and Gontijo Lopes, Raphael and Karagol Ayan, Burcu and Salimans, Tim and others},
  journal={Advances in neural information processing systems},
  volume={35},
  pages={36479--36494},
  year={2022}
}

@article{lin2024awq,
  title={Awq: Activation-aware weight quantization for on-device llm compression and acceleration},
  author={Lin, Ji and Tang, Jiaming and Tang, Haotian and Yang, Shang and Chen, Wei-Ming and Wang, Wei-Chen and Xiao, Guangxuan and Dang, Xingyu and Gan, Chuang and Han, Song},
  journal={Proceedings of machine learning and systems},
  volume={6},
  pages={87--100},
  year={2024}
}

@article{frantar2022gptq,
  title={Gptq: Accurate post-training quantization for generative pre-trained transformers},
  author={Frantar, Elias and Ashkboos, Saleh and Hoefler, Torsten and Alistarh, Dan},
  journal={arXiv preprint arXiv:2210.17323},
  year={2022}
}

@inproceedings{chen2025q,
  title={Q-dit: Accurate post-training quantization for diffusion transformers},
  author={Chen, Lei and Meng, Yuan and Tang, Chen and Ma, Xinzhu and Jiang, Jingyan and Wang, Xin and Wang, Zhi and Zhu, Wenwu},
  booktitle={Proceedings of the Computer Vision and Pattern Recognition Conference},
  pages={28306--28315},
  year={2025}
}

@article{wu2024ptq4dit,
  title={Ptq4dit: Post-training quantization for diffusion transformers},
  author={Wu, Junyi and Wang, Haoxuan and Shang, Yuzhang and Shah, Mubarak and Yan, Yan},
  journal={Advances in neural information processing systems},
  volume={37},
  pages={62732--62755},
  year={2024}
}

@article{chee2023quip,
  title={Quip: 2-bit quantization of large language models with guarantees},
  author={Chee, Jerry and Cai, Yaohui and Kuleshov, Volodymyr and De Sa, Christopher M},
  journal={Advances in neural information processing systems},
  volume={36},
  pages={4396--4429},
  year={2023}
}

@inproceedings{liu2025spinquant,
  title={Spinquant: Llm quantization with learned rotations},
  author={Liu, Zechun and Zhao, Changsheng and Fedorov, Igor and Soran, Bilge and Choudhary, Dhruv and Krishnamoorthi, Raghuraman and Chandra, Vikas and Tian, Yuandong and Blankevoort, Tijmen},
  booktitle={International Conference on Learning Representations},
  volume={2025},
  pages={92009--92032},
  year={2025}
}

@article{li2024svdquant,
  title={Svdquant: Absorbing outliers by low-rank components for 4-bit diffusion models},
  author={Li, Muyang and Lin, Yujun and Zhang, Zhekai and Cai, Tianle and Li, Xiuyu and Guo, Junxian and Xie, Enze and Meng, Chenlin and Zhu, Jun-Yan and Han, Song},
  journal={arXiv preprint arXiv:2411.05007},
  year={2024}
}

@article{zhao2024vidit,
  title={Vidit-q: Efficient and accurate quantization of diffusion transformers for image and video generation},
  author={Zhao, Tianchen and Fang, Tongcheng and Huang, Haofeng and Liu, Enshu and Wan, Rui and Soedarmadji, Widyadewi and Li, Shiyao and Lin, Zinan and Dai, Guohao and Yan, Shengen and others},
  journal={arXiv preprint arXiv:2406.02540},
  year={2024}
}

@article{tseng2024quip,
  title={Quip\#: Even better llm quantization with hadamard incoherence and lattice codebooks},
  author={Tseng, Albert and Chee, Jerry and Sun, Qingyao and Kuleshov, Volodymyr and De Sa, Christopher},
  journal={Proceedings of machine learning research},
  volume={235},
  pages={48630},
  year={2024}
}

@article{zhang2026adatsq,
  title={AdaTSQ: Pushing the Pareto Frontier of Diffusion Transformers via Temporal-Sensitivity Quantization},
  author={Zhang, Shaoqiu and Ding, Zizhong and Yang, Kaicheng and Wu, Junyi and Yan, Xianglong and Li, Xi and Duan, Bingnan and Fang, Jianping and Zhang, Yulun},
  journal={arXiv preprint arXiv:2602.09883},
  year={2026}
}

@article{zandieh2025turboquant,
  title={Turboquant: Online vector quantization with near-optimal distortion rate},
  author={Zandieh, Amir and Daliri, Majid and Hadian, Majid and Mirrokni, Vahab},
  journal={arXiv preprint arXiv:2504.19874},
  year={2025}
}

@article{ho2020denoising,
  title={Denoising diffusion probabilistic models},
  author={Ho, Jonathan and Jain, Ajay and Abbeel, Pieter},
  journal={Advances in neural information processing systems},
  volume={33},
  pages={6840--6851},
  year={2020}
}

@inproceedings{chen2024pixart,
  title={Pixart-$\sigma$: Weak-to-strong training of diffusion transformer for 4k text-to-image generation},
  author={Chen, Junsong and Ge, Chongjian and Xie, Enze and Wu, Yue and Yao, Lewei and Ren, Xiaozhe and Wang, Zhongdao and Luo, Ping and Lu, Huchuan and Li, Zhenguo},
  booktitle={European Conference on Computer Vision},
  pages={74--91},
  year={2024},
  organization={Springer}
}

@misc{labs2025flux1kontextflowmatching,
      title={FLUX.1 Kontext: Flow Matching for In-Context Image Generation and Editing in Latent Space},
      author={Black Forest Labs and Stephen Batifol and Andreas Blattmann and Frederic Boesel and Saksham Consul and Cyril Diagne and Tim Dockhorn and Jack English and Zion English and Patrick Esser and Sumith Kulal and Kyle Lacey and Yam Levi and Cheng Li and Dominik Lorenz and Jonas Müller and Dustin Podell and Robin Rombach and Harry Saini and Axel Sauer and Luke Smith},
      year={2025},
      eprint={2506.15742},
      archivePrefix={arXiv},
      primaryClass={cs.GR},
      url={https://arxiv.org/abs/2506.15742},
}

@article{xie2024sana,
  title={Sana: Efficient high-resolution image synthesis with linear diffusion transformers},
  author={Xie, Enze and Chen, Junsong and Chen, Junyu and Cai, Han and Tang, Haotian and Lin, Yujun and Zhang, Zhekai and Li, Muyang and Zhu, Ligeng and Lu, Yao and others},
  journal={arXiv preprint arXiv:2410.10629},
  year={2024}
}

@article{wu2025qwen,
  title={Qwen-image technical report},
  author={Wu, Chenfei and Li, Jiahao and Zhou, Jingren and Lin, Junyang and Gao, Kaiyuan and Yan, Kun and Yin, Sheng-ming and Bai, Shuai and Xu, Xiao and Chen, Yilei and others},
  journal={arXiv preprint arXiv:2508.02324},
  year={2025}
}

@article{cai2025z,
  title={Z-image: An efficient image generation foundation model with single-stream diffusion transformer},
  author={Cai, Huanqia and Cao, Sihan and Du, Ruoyi and Gao, Peng and Hoi, Steven and Hou, Zhaohui and Huang, Shijie and Jiang, Dengyang and Jin, Xin and Li, Liangchen and others},
  journal={arXiv preprint arXiv:2511.22699},
  year={2025}
}

@inproceedings{yang2025cogvideox,
  title={Cogvideox: Text-to-video diffusion models with an expert transformer},
  author={Yang, Zhuoyi and Teng, Jiayan and Zheng, Wendi and Ding, Ming and Huang, Shiyu and Xu, Jiazheng and Yang, Yuanming and Hong, Wenyi and Zhang, Xiaohan and Feng, Guanyu and others},
  booktitle={International Conference on Learning Representations},
  volume={2025},
  pages={83048--83077},
  year={2025}
}

@article{kong2024hunyuanvideo,
  title={Hunyuanvideo: A systematic framework for large video generative models},
  author={Kong, Weijie and Tian, Qi and Zhang, Zijian and Min, Rox and Dai, Zuozhuo and Zhou, Jin and Xiong, Jiangfeng and Li, Xin and Wu, Bo and Zhang, Jianwei and others},
  journal={arXiv preprint arXiv:2412.03603},
  year={2024}
}

@misc{sora2024,
  author={OpenAI},
  title={Sora: Creating Video from Text},
  year={2024},
  howpublished={\url{https://openai.com/sora}},
  note={Accessed: 2024-02-15},
}

@article{wan2025wan,
  title={Wan: Open and advanced large-scale video generative models},
  author={Wan, Team and Wang, Ang and Ai, Baole and Wen, Bin and Mao, Chaojie and Xie, Chen-Wei and Chen, Di and Yu, Feiwu and Zhao, Haiming and Yang, Jianxiao and others},
  journal={arXiv preprint arXiv:2503.20314},
  year={2025}
}

@article{huang2025qvgen,
  title={Qvgen: Pushing the limit of quantized video generative models},
  author={Huang, Yushi and Gong, Ruihao and Liu, Jing and Ding, Yifu and Lv, Chengtao and Qin, Haotong and Zhang, Jun},
  journal={arXiv preprint arXiv:2505.11497},
  year={2025}
}

@article{mezzadri2006generate,
  title={How to generate random matrices from the classical compact groups},
  author={Mezzadri, Francesco},
  journal={arXiv preprint math-ph/0609050},
  year={2006}
}

@article{lloyd1982least,
  title={Least squares quantization in PCM},
  author={Lloyd, Stuart},
  journal={IEEE transactions on information theory},
  volume={28},
  number={2},
  pages={129--137},
  year={1982},
  publisher={IEEE}
}

@article{max1960quantizing,
  title={Quantizing for minimum distortion},
  author={Max, Joel},
  journal={IRE Transactions on Information Theory},
  volume={6},
  number={1},
  pages={7--12},
  year={1960},
  publisher={IEEE}
}

@article{ailon2009fast,
  title={The fast Johnson--Lindenstrauss transform and approximate nearest neighbors},
  author={Ailon, Nir and Chazelle, Bernard},
  journal={SIAM Journal on computing},
  volume={39},
  number={1},
  pages={302--322},
  year={2009},
  publisher={SIAM}
}

@article{tropp2011improved,
  title={Improved analysis of the subsampled randomized Hadamard transform},
  author={Tropp, Joel A},
  journal={Advances in Adaptive Data Analysis},
  volume={3},
  number={01n02},
  pages={115--126},
  year={2011},
  publisher={World Scientific}
}

@article{li2023qdm,
  title={Q-dm: An efficient low-bit quantized diffusion model},
  author={Li, Yanjing and Xu, Sheng and Cao, Xianbin and Sun, Xiao and Zhang, Baochang},
  journal={Advances in neural information processing systems},
  volume={36},
  pages={76680--76691},
  year={2023}
}

@inproceedings{he2024efficientdm,
  title={Efficientdm: Efficient quantization-aware fine-tuning of low-bit diffusion models},
  author={He, Yefei and Liu, Jing and Wu, Weijia and Zhou, Hong and Zhuang, Bohan},
  booktitle={International Conference on Learning Representations},
  volume={2024},
  pages={15731--15750},
  year={2024}
}

@article{li2025dvd,
  title={DVD-Quant: Data-free Video Diffusion Transformers Quantization},
  author={Li, Zhiteng and Li, Hanxuan and Wu, Junyi and Liu, Kai and Qin, Haotong and Kong, Linghe and Chen, Guihai and Zhang, Yulun and Yang, Xiaokang},
  journal={arXiv preprint arXiv:2505.18663},
  year={2025}
}

@article{sun2024flatquant,
  title={Flatquant: Flatness matters for llm quantization},
  author={Sun, Yuxuan and Liu, Ruikang and Bai, Haoli and Bao, Han and Zhao, Kang and Li, Yuening and Hu, Jiaxin and Yu, Xianzhi and Hou, Lu and Yuan, Chun and others},
  journal={arXiv preprint arXiv:2410.09426},
  year={2024}
}

@article{hu2025ostquant,
  title={Ostquant: Refining large language model quantization with orthogonal and scaling transformations for better distribution fitting},
  author={Hu, Xing and Cheng, Yuan and Yang, Dawei and Xu, Zukang and Yuan, Zhihang and Yu, Jiangyong and Xu, Chen and Jiang, Zhe and Zhou, Sifan},
  journal={arXiv preprint arXiv:2501.13987},
  year={2025}
}

@article{lin2024duquant,
  title={Duquant: Distributing outliers via dual transformation makes stronger quantized llms},
  author={Lin, Haokun and Xu, Haobo and Wu, Yichen and Cui, Jingzhi and Zhang, Yingtao and Mou, Linzhan and Song, Linqi and Sun, Zhenan and Wei, Ying},
  journal={Advances in Neural Information Processing Systems},
  volume={37},
  pages={87766--87800},
  year={2024}
}

@article{gu2026lopro,
  title={LoPRo: Enhancing Low-Rank Quantization via Permuted Block-Wise Rotation},
  author={Gu, Hongyaoxing and Hu, Lijuan and Yu, Liye and Li, Haowei and Liu, Fangfang},
  journal={arXiv preprint arXiv:2601.19675},
  year={2026}
}

@article{sanjeet2026pushing,
  title={Pushing the Limits of Block Rotations in Post-Training Quantization},
  author={Sanjeet, Sai and Colbert, Ian and Monteagudo-Lago, Pablo and Franco, Giuseppe and Umuroglu, Yaman and Fraser, Nicholas J},
  journal={arXiv preprint arXiv:2601.22347},
  year={2026}
}

@article{yang2025lrq,
  title={LRQ-DiT: Log-Rotation Post-Training Quantization of Diffusion Transformers for Image and Video Generation},
  author={Yang, Lianwei and Lin, Haokun and Zhao, Tianchen and Wu, Yichen and Zhu, Hongyu and Xie, Ruiqi and Sun, Zhenan and Wang, Yu and Gu, Qingyi},
  journal={arXiv preprint arXiv:2508.03485},
  year={2025}
}

@article{huang2025convrot,
  title={ConvRot: Rotation-Based Plug-and-Play 4-bit Quantization for Diffusion Transformers},
  author={Huang, Feice and Han, Zuliang and Zhou, Xing and Chen, Yihuang and Zhu, Lifei and Wang, Haoqian},
  journal={arXiv preprint arXiv:2512.03673},
  year={2025}
}

@inproceedings{han2026polarquant,
  title={PolarQuant: Vector Quantization with Polar Transformation},
  author={Han, Insu and Kacham, Praneeth and Karbasi, Amin and Mirrokni, Vahab and Zandieh, Amir},
  booktitle={The 29th International Conference on Artificial Intelligence and Statistics},
  year={2026}
}

@article{cheng2026permuquant,
  title={PermuQuant: Lowering Per-Group Quantization Error by Reordering Channels for Diffusion Models},
  author={Cheng, Yongsen and Liu, Kai and Tao, Kaiwen and Li, Junxian and Wang, Zhixin and Chen, Zhikai and Pei, Renjing and Zhang, Yulun},
  journal={arXiv preprint arXiv:2605.09503},
  year={2026}
}

@article{dettmers2023spqr,
  title={Spqr: A sparse-quantized representation for near-lossless llm weight compression},
  author={Dettmers, Tim and Svirschevski, Ruslan and Egiazarian, Vage and Kuznedelev, Denis and Frantar, Elias and Ashkboos, Saleh and Borzunov, Alexander and Hoefler, Torsten and Alistarh, Dan},
  journal={arXiv preprint arXiv:2306.03078},
  year={2023}
}

@article{kim2023squeezellm,
  title={Squeezellm: Dense-and-sparse quantization},
  author={Kim, Sehoon and Hooper, Coleman and Gholami, Amir and Dong, Zhen and Li, Xiuyu and Shen, Sheng and Mahoney, Michael W and Keutzer, Kurt},
  journal={arXiv preprint arXiv:2306.07629},
  year={2023}
}

@inproceedings{park2024lut,
  title={Lut-gemm: Quantized matrix multiplication based on luts for efficient inference in large-scale generative language models},
  author={Park, Gunho and Kim, Minsub and Lee, Sungjae and Kim, Jeonghoon and Kwon, Beomseok and Kwon, Se Jung and Kim, Byeongwook and Lee, Youngjoo and Lee, Dongsoo and others},
  booktitle={International Conference on Learning Representations},
  volume={2024},
  pages={38069--38086},
  year={2024}
}

@inproceedings{guo2024fast,
  title={Fast matrix multiplications for lookup table-quantized llms},
  author={Guo, Han and Brandon, William and Cholakov, Radostin and Ragan-Kelley, Jonathan and Xing, Eric and Kim, Yoon},
  booktitle={Findings of the Association for Computational Linguistics: EMNLP 2024},
  pages={12419--12433},
  year={2024}
}

@article{feng2025s,
  title={S${}^{2}$ Q-VDiT: Accurate Quantized Video Diffusion Transformer with Salient Data and Sparse Token Distillation},
  author={Feng, Weilun and Qin, Haotong and Yang, Chuanguang and Li, Xiangqi and Yang, Han and Li, Yuqi and An, Zhulin and Huang, Libo and Magno, Michele and Xu, Yongjun},
  journal={arXiv preprint arXiv:2508.04016},
  year={2025}
}

@article{yang2025robuq,
  title={RobuQ: Pushing DiTs to W1. 58A2 via Robust Activation Quantization},
  author={Yang, Kaicheng and Zhang, Xun and Qin, Haotong and Lin, Yucheng and Yang, Kaisen and Yan, Xianglong and Zhang, Yulun},
  journal={arXiv preprint arXiv:2509.23582},
  year={2025}
}
}

\clearpage
\setcounter{page}{1}
\maketitlesupplementary
\appendix

\section{Proof Sketch for RPBH Incoherence}
\label{sec:supp-proof-srht}


\paragraph{Setup.}
Fix a unit vector $\tilde{\mathbf{x}} \in \mathbb{R}^d$ and write $d = kh$. Let $\mathbf{y} = \mathbf{P}_\pi \tilde{\mathbf{x}}$ have blocks $\mathbf{y}^{(j)} \in \mathbb{R}^h$ with masses $M_j = \|\mathbf{y}^{(j)}\|_2^2$ summing to $1$, outputs $\mathbf{z}^{(j)} = \mathbf{H}_h \mathbf{D}_j \mathbf{y}^{(j)}$ with $(\mathbf{H}_h)_{li} = \pm 1/\sqrt{h}$ and $\mathbf{D}_j$ Rademacher, and $\mu_\infty = \|\tilde{\mathbf{x}}\|_\infty^2$. Write $\mathbf{z} = (\mathbf{z}^{(1)}, \dots, \mathbf{z}^{(k)})$ for the full output $\boldsymbol{\Pi}_d \tilde{\mathbf{x}}$.
\begin{lemma}[Per-block incoherence]
\label{lem:block-incoherence}
For any fixed partition (any $\pi$), with probability at least $1 - \delta/2$ over $\{\mathbf{D}_j\}$,
\begin{equation}
\|\mathbf{z}\|_\infty \le \sqrt{2 \log(4d/\delta)/h}.
\label{eq:lem1}
\end{equation}
\end{lemma}
\begin{proof}
Each output coordinate $z^{(j)}_l = \sum_i (\mathbf{H}_h)_{li}\, \sigma^{(j)}_i\, y^{(j)}_i$ is a mean-zero Rademacher sum with variance $\sum_i (\mathbf{H}_h)_{li}^2 (y^{(j)}_i)^2 = M_j/h \le 1/h$. Hoeffding gives $\Pr[|z^{(j)}_l| > t] \le 2 e^{-t^2 h/2}$, and a union bound over the $d$ coordinates yields Equation~\eqref{eq:lem1}.
\end{proof}
\begin{lemma}[Mass balancing]
\label{lem:balancing}
With probability at least $1 - \delta/2$ over $\pi$, for all $j$,
\begin{equation}
\big| M_j - \tfrac{1}{k} \big| \le \mu_\infty \sqrt{(h/2)\log(4k/\delta)}.
\label{eq:lem2}
\end{equation}
\end{lemma}
\begin{proof}
Each $M_j$ is a sum of $h$ values drawn without replacement from $\{\tilde{x}_i^2\} \subseteq [0, \mu_\infty]$ with mean $1/k$. By Hoeffding's bound for sampling without replacement, $\Pr[|M_j - 1/k| \ge \epsilon] \le 2 e^{-2\epsilon^2 / (h\mu_\infty^2)}$, and a union bound over the $k$ blocks gives Equation~\eqref{eq:lem2}.
\end{proof}
\medskip
\noindent\textit{Proposition~\ref{prop:marginal} (restated).}\\
Let $\rho = d\,\mu_\infty \sqrt{(1/2h)\log(4k/\delta)}$. With probability at least $1-\delta$ over $\boldsymbol{\Pi}_d$, every coordinate of $\mathbf{z} = \boldsymbol{\Pi}_d\tilde{\mathbf{x}}$ is mean-zero with conditional variance $\mathrm{Var}(z_i \mid \pi) \in \frac{1}{d}(1 \pm \rho)$, and
\begin{equation}
\|\boldsymbol{\Pi}_d \tilde{\mathbf{x}}\|_\infty \le \sqrt{\tfrac{2}{d}(1 + \rho)\log(4d/\delta)}.
\label{eq:thm}
\end{equation}
\begin{proof}
Each $z_i$ is a mean-zero Rademacher sum, so $\mathbb{E}[z_i] = 0$. On the event of Lemma~\ref{lem:balancing}, $M_j \in \tfrac{1}{k}(1 \pm \rho)$, so each coordinate has variance $M_j/h \in \tfrac{1}{d}(1 \pm \rho)$. Equation~\eqref{eq:thm} then follows by repeating the proof of Lemma~\ref{lem:block-incoherence} with $M_j/h \le \tfrac{1}{d}(1 + \rho)$, after a union bound over the two events, each holding with probability at least $1 - \delta/2$.
\end{proof}
\begin{remark}
Lemma~\ref{lem:block-incoherence} uses only $M_j \le 1$, so it holds with or without the permutation. The permutation enters through Lemma~\ref{lem:balancing} alone, equalizing every per-coordinate variance to $1/d$ regardless of how outlier channels fall into blocks, which is what the no-permutation variant loses at low bit-width (Section~\ref{sec:ablation:rotation}).
\end{remark}

\begin{remark}
Variance concentration upgrades to a quantitative Gaussian approximation. Conditional on $\pi$, each coordinate $z_i$ in block $j$ is a sum of $h$ independent bounded terms with total variance $M_j/h$, so the Berry--Esseen inequality bounds its Kolmogorov distance to $\mathcal{N}(0, M_j/h)$ by $C \sum_i |y^{(j)}_i|^3 / M_j^{3/2} \le C\sqrt{\mu_\infty / M_j}$ for a universal constant $C$. On the event of Lemma~\ref{lem:balancing} this is at most $C\sqrt{\mu_\infty k/(1-\rho)}$, so whenever no coordinate carries an outsized share of the norm, every rotated coordinate is close to $\mathcal{N}(0, 1/d)$ in distribution, not only in variance. Figure~\ref{fig:distribution} confirms this empirically.
\end{remark}

\section{Additional Experimental Details}
\label{sec:supp-additional}

\subsection{Generation settings}
\label{sec:supp-gen}
Image models use the sampler and step count of their public checkpoints, FLUX.1-schnell at 4 steps and guidance 0.0, FLUX.1-dev at 50 steps and guidance 3.5, and Z-Image-Turbo at 10 steps and guidance 0.0. Video models use Wan~2.1-1.3B at 81 frames, $480{\times}832$, 50 steps, CFG 5.0, and CogVideoX-2B at 49 frames, $480{\times}720$, 50 steps, CFG 6.0. NVIDIA H100 GPUs are used for experiments.

\subsection{Quantized and skipped layers}
\label{sec:supp-layers}
We quantize every linear projection in the transformer block through the OrbitQuant path, namely the image- and text-side $Q$, $K$, $V$ and output projections and the feed-forward layers of every block, including the text-conditioning $K$ and $V$ projections that consume text-encoder hidden states (the joint-attention text path in FLUX and Z-Image, the cross-attention projections in Wan and CogVideoX). 
AdaLN modulation projections are the one exception. Their output parameterizes a timestep-dependent elementwise scale-and-shift. A static norm affine can be folded into neighboring weights, as rotation-based LLM quantizers do~\cite{ashkboos2024quarot}, but this dynamic modulation cannot. The shared-rotation cancellation of Section~\ref{sec:method} therefore has no counterpart here. Their input is also a single conditioning token per step, leaving no activation compute to save. We therefore quantize only their weights, with INT4 RTN at group size 64 and BF16 activations. Embeddings, the timestep MLP, the final un-patchify head, and the text encoder stay in BF16.

\begin{table*}[t]
\caption{GenEval results at the lowest bit-widths, W3A3 and W2A3, on three image diffusion transformers. \textbf{Bold} and \underline{underlined} entries indicate the best and second-best PTQ result within each (model, bit-width) group. $\uparrow$ means higher is better.$\dagger$ represents our implementation.}
\label{tab:geneval-low}
\centering
\resizebox{0.95\textwidth}{!}{%
\begin{tabular}{clcccccccc}
\toprule
Model & Method & Bit & Single object~$\uparrow$ & Two object~$\uparrow$ & Counting~$\uparrow$ & Colors~$\uparrow$ & Position~$\uparrow$ & Color attribution~$\uparrow$ & Overall~$\uparrow$ \\
\midrule
\multirow{8}{*}{FLUX.1-schnell}
  & FP16                                  & 16/16 & 0.997 & 0.884 & 0.600 & 0.742 & 0.275 & 0.488 & 0.664 \\
  \cdashline{2-10}\\[-1.75ex]
  & SVDQuant~\cite{li2024svdquant}        & W3A3 & 0.820 & 0.647 & 0.466 & 0.560 & 0.160 & 0.373 & 0.504 \\
  & AdaTSQ~\cite{zhang2026adatsq}         & W3A3 & \textbf{0.997} & \textbf{0.920} & \underline{0.530} & \underline{0.688} & \textbf{0.230} & \underline{0.440} & \underline{0.634} \\
  & \cellcolor{gray!25}OrbitQuant            & \cellcolor{gray!25}W3A3 & \cellcolor{gray!25}\underline{0.978} & \cellcolor{gray!25}\underline{0.861} & \cellcolor{gray!25}\textbf{0.684} & \cellcolor{gray!25}\textbf{0.777} & \cellcolor{gray!25}\underline{0.223} & \cellcolor{gray!25}\textbf{0.542} & \cellcolor{gray!25}\textbf{0.678} \\
  \cdashline{2-10}\\[-1.75ex]
  & QuaRot$\dagger$~\cite{ashkboos2024quarot}      & W2A3 & 0.003 & 0.000 & 0.000 & 0.000 & 0.000 & 0.000 & 0.001 \\
  & SmoothQuant$\dagger$~\cite{xiao2023smoothquant} & W2A3 & 0.003 & 0.000 & 0.000 & 0.000 & 0.000 & 0.000 & 0.001 \\
  & ViDiT-Q$\dagger$~\cite{zhao2024vidit}          & W2A3 & 0.009 & 0.000 & 0.000 & 0.003 & 0.000 & 0.000 & 0.002 \\
  & \cellcolor{gray!25}OrbitQuant            & \cellcolor{gray!25}W2A3 & \cellcolor{gray!25}\textbf{0.947} & \cellcolor{gray!25}\textbf{0.573} & \cellcolor{gray!25}\textbf{0.431} & \cellcolor{gray!25}\textbf{0.691} & \cellcolor{gray!25}\textbf{0.140} & \cellcolor{gray!25}\textbf{0.318} & \cellcolor{gray!25}\textbf{0.517} \\
\midrule
\multirow{8}{*}{FLUX.1-dev}
  & FP16                                  & 16/16 & 0.984 & 0.823 & 0.769 & 0.771 & 0.203 & 0.450 & 0.667 \\
  \cdashline{2-10}\\[-1.75ex]
  & SVDQuant~\cite{li2024svdquant}        & W3A3 & 0.869 & 0.288 & 0.425 & 0.524 & 0.033 & 0.123 & 0.377 \\
  & AdaTSQ~\cite{zhang2026adatsq}         & W3A3 & \underline{0.956} & \underline{0.548} & \textbf{0.628} & \underline{0.656} & \underline{0.083} & \underline{0.290} & \underline{0.527} \\
  & \cellcolor{gray!25}OrbitQuant            & \cellcolor{gray!25}W3A3 & \cellcolor{gray!25}\textbf{0.981} & \cellcolor{gray!25}\textbf{0.684} & \cellcolor{gray!25}\underline{0.606} & \cellcolor{gray!25}\textbf{0.734} & \cellcolor{gray!25}\textbf{0.128} & \cellcolor{gray!25}\textbf{0.372} & \cellcolor{gray!25}\textbf{0.584} \\
  \cdashline{2-10}\\[-1.75ex]
  & QuaRot$\dagger$~\cite{ashkboos2024quarot}      & W2A3 & 0.003 & 0.000 & 0.000 & 0.000 & 0.000 & 0.000 & 0.001 \\
  & SmoothQuant$\dagger$~\cite{xiao2023smoothquant} & W2A3 & 0.003 & 0.000 & 0.000 & 0.000 & 0.000 & 0.000 & 0.001 \\
  & ViDiT-Q$\dagger$~\cite{zhao2024vidit}          & W2A3 & 0.0013 & 0.000 & 0.000 & 0.003 & 0.000 & 0.000 & 0.002 \\
  & \cellcolor{gray!25}OrbitQuant            & \cellcolor{gray!25}W2A3 & \cellcolor{gray!25}\textbf{0.906} & \cellcolor{gray!25}\textbf{0.235} & \cellcolor{gray!25}\textbf{0.338} & \cellcolor{gray!25}\textbf{0.582} & \cellcolor{gray!25}\textbf{0.050} & \cellcolor{gray!25}\textbf{0.120} & \cellcolor{gray!25}\textbf{0.372} \\
\midrule
\multirow{8}{*}{Z-Image-Turbo}
  & FP16                                  & 16/16 & 1.000 & 0.907 & 0.709 & 0.859 & 0.468 & 0.583 & 0.754 \\
  \cdashline{2-10}\\[-1.75ex]
  & SVDQuant~\cite{li2024svdquant}        & W3A3 & 0.005 & 0.000 & 0.000 & 0.000 & 0.000 & 0.000 & 0.000 \\
  & AdaTSQ~\cite{zhang2026adatsq}         & W3A3 & \textbf{0.994} & \textbf{0.870} & \underline{0.550} & \textbf{0.885} & \textbf{0.410} & \underline{0.455} & \underline{0.694} \\
  & \cellcolor{gray!25}OrbitQuant            & \cellcolor{gray!25}W3A3 & \cellcolor{gray!25}\textbf{0.994} & \cellcolor{gray!25}\underline{0.846} & \cellcolor{gray!25}\textbf{0.750} & \cellcolor{gray!25}\underline{0.859} & \cellcolor{gray!25}\underline{0.395} & \cellcolor{gray!25}\textbf{0.598} & \cellcolor{gray!25}\textbf{0.740} \\
  \cdashline{2-10}\\[-1.75ex]
  & QuaRot$\dagger$~\cite{ashkboos2024quarot}      & W2A3 & 0.013 & 0.000 & 0.000 & 0.000 & 0.000 & 0.000 & 0.002 \\
  & SmoothQuant$\dagger$~\cite{xiao2023smoothquant} & W2A3 & 0.009 & 0.000 & 0.000 & 0.000 & 0.000 & 0.000 & 0.002 \\
  & ViDiT-Q$\dagger$~\cite{zhao2024vidit}          & W2A3 & 0.000 & 0.000 & 0.000 & 0.003 & 0.000 & 0.000 & 0.000 \\
  & \cellcolor{gray!25}OrbitQuant            & \cellcolor{gray!25}W2A3 & \cellcolor{gray!25}\textbf{0.272} & \cellcolor{gray!25}\textbf{0.023} & \cellcolor{gray!25}\textbf{0.028} & \cellcolor{gray!25}\textbf{0.269} & \cellcolor{gray!25}\textbf{0.018} & \cellcolor{gray!25}\textbf{0.023} & \cellcolor{gray!25}\textbf{0.105} \\
\bottomrule
\end{tabular}}
\end{table*}

\begin{table*}[t]
\caption{VBench results on Wan~14B at W4A4. Per-dimension scores over eight VBench dimensions. \textbf{Bold} and \underline{underlined} entries indicate the best and second-best PTQ result. QVGen is a QAT method, shown for reference and excluded from the PTQ ranking. $\uparrow$ means higher is better. $\dagger$ represents our implementation.}
\label{tab:vbench-wan14b}
\centering
\resizebox{\textwidth}{!}{%
\begin{tabular}{clcccccccccc}
\toprule
Model & Method & Bit & \makecell{Imaging\\Quality~$\uparrow$} & \makecell{Aesthetic\\Quality~$\uparrow$} & \makecell{Motion\\Smoothness~$\uparrow$} & \makecell{Dynamic\\Degree~$\uparrow$} & \makecell{Background\\Consistency~$\uparrow$} & \makecell{Subject\\Consistency~$\uparrow$} & Scene~$\uparrow$ & \makecell{Overall\\Consistency~$\uparrow$} \\
\midrule
\multirow{5}{*}{Wan~14B}
  & BF16                                    & 16/16 & 0.6514 & 0.6136 & 0.9738 & 0.7389 & 0.9632 & 0.9365 & 0.3330 & 0.2629 \\
  \cdashline{2-12}\\[-1.75ex]
  
  & SmoothQuant$\dagger$~\cite{xiao2023smoothquant}  & W4A4 & 0.5971 & 0.5263 & \textbf{0.9763} & 0.4472 & 0.9390 & 0.9171 & 0.1439 & 0.2327 \\
  & QuaRot$\dagger$~\cite{ashkboos2024quarot}      & W4A4 & \underline{0.6332} & \underline{0.5686} & 0.9701 & \underline{0.5500} & 0.9504 & 0.9185 & \underline{0.2589} & \underline{0.2541} \\
  & ViDiT-Q$\dagger$~\cite{zhao2024vidit}          & W4A4 & 0.5948 & 0.5373 & 0.9672 & 0.5417 & \underline{0.9533} & \underline{0.9202} & 0.1849 & 0.2432 \\
  & \cellcolor{gray!25}OrbitQuant            & \cellcolor{gray!25}W4A4 & \cellcolor{gray!25}\textbf{0.6405} & \cellcolor{gray!25}\textbf{0.6022} & \cellcolor{gray!25}\underline{0.9754} & \cellcolor{gray!25}\textbf{0.6250} & \cellcolor{gray!25}\textbf{0.9559} & \cellcolor{gray!25}\textbf{0.9363} & \cellcolor{gray!25}\textbf{0.3285} & \cellcolor{gray!25}\textbf{0.2615} \\
\midrule
\multirow{6}{*}{HunyuanVideo}
  & BF16                                                 & 16/16 & 0.6478 & 0.6253 & 0.9930 & 0.5139 & 0.9701 & 0.9605 & 0.4281 & 0.2586 \\
  \cdashline{2-12}\\[-1.75ex]
  & SmoothQuant~\cite{xiao2023smoothquant}    & W4A4  & 0.5946 & 0.4841 & 0.9879 & 0.0139 & 0.9672 & 0.9497 & 0.0784 & 0.2109 \\
  & QuaRot~\cite{ashkboos2024quarot}          & W4A4  & 0.5430 & 0.4485 & 0.9222 & \textbf{0.8750} & \underline{0.9769} & 0.9264 & 0.0094 & 0.1733 \\
  & ViDiT-Q~\cite{zhao2024vidit}              & W4A4  & 0.4010 & 0.4536 & \textbf{0.9943} & 0.0000 & 0.9719 & \textbf{0.9729} & 0.0785 & 0.1966 \\
  & DVD-Quant~\cite{li2025dvd}           & W4A4  & \underline{0.6182} & \textbf{0.6196} & 0.9915 & \underline{0.5694} & \textbf{0.9782} & \underline{0.9661} & \underline{0.2994} & \textbf{0.2568} \\
  & \cellcolor{gray!25}OrbitQuant            & \cellcolor{gray!25}W4A4 & \cellcolor{gray!25}\textbf{0.6209} & \cellcolor{gray!25}\underline{0.6072} & \cellcolor{gray!25}\underline{0.9930}& \cellcolor{gray!25}0.4417& \cellcolor{gray!25}0.9751 & \cellcolor{gray!25}0.9622 & \cellcolor{gray!25}\textbf{0.3052} & \cellcolor{gray!25}\underline{0.2283} \\
\bottomrule
\end{tabular}}
\end{table*}

\section{Additional Experiments}
\label{supple:add-exp}
\subsection{Lowest bit-widths: W3A3 and W2A3}
\label{supple:low_bit}
We push to the lowest bit-widths, W3A3 and W2A3, on three image diffusion transformers. Table~\ref{tab:geneval-low} reports GenEval. At W3A3 we compare against the low-bit image quantizers SVDQuant~\cite{li2024svdquant} and AdaTSQ~\cite{zhang2026adatsq}, and at W2A3 against the rotation and smoothing baselines.
At W3A3 OrbitQuant has the best Overall on all three models and stays close to FP16. AdaTSQ is competitive and leads on a few individual dimensions, but OrbitQuant is the most consistent across them, while SVDQuant collapses entirely on Z-Image-Turbo.
W2A3 is the harder test. The rotation and smoothing baselines collapse to near zero on every model, since a 3-bit uniform grid cannot place its levels where the rotated activations are dense. OrbitQuant is the only method that stays functional, remaining usable on the FLUX models. Z-Image-Turbo is the exception, where even OrbitQuant degrades sharply, marking the limit of a calibration-free codebook at this bit-width.

\begin{table*}[t]
\caption{Seed robustness of OrbitQuant on GenEval. Mean and standard deviation over three random seeds on three image diffusion transformers at W4A4 and W2A4. $\uparrow$ means higher is better.}
\label{tab:robustness}
\centering
\resizebox{\textwidth}{!}{%
\begin{tabular}{llccccccc}
\toprule
Model & Bit & Single object~$\uparrow$ & Two object~$\uparrow$ & Counting~$\uparrow$ & Colors~$\uparrow$ & Position~$\uparrow$ & Color attribution~$\uparrow$ & \textbf{Overall}~$\uparrow$  \\
\midrule
\multirow{2}{*}{FLUX.1-schnell}
  & W4A4 & $0.991 \pm 0.003$ & $0.879 \pm 0.019$ & $0.685 \pm 0.018$ & $0.793 \pm 0.011$ & $0.280 \pm 0.039$ & $0.510 \pm 0.014$ & $\mathbf{0.690 \pm 0.012}$  \\
  & W2A4 & $0.963 \pm 0.011$ & $0.692 \pm 0.013$ & $0.577 \pm 0.010$ & $0.754 \pm 0.025$ & $0.164 \pm 0.038$ & $0.423 \pm 0.031$ & $\mathbf{0.595 \pm 0.008}$  \\
\midrule
\multirow{2}{*}{FLUX.1-dev}
  & W4A4 & $0.990 \pm 0.002$ & $0.763 \pm 0.005$ & $0.721 \pm 0.027$ & $0.761 \pm 0.007$ & $0.177 \pm 0.010$ & $0.421 \pm 0.009$ & $\mathbf{0.639 \pm 0.004}$ \\
  & W2A4 & $0.943 \pm 0.014$ & $0.395 \pm 0.032$ & $0.480 \pm 0.011$ & $0.668 \pm 0.012$ & $0.079 \pm 0.027$ & $0.198 \pm 0.007$ & $\mathbf{0.460 \pm 0.014}$ \\
\midrule
\multirow{2}{*}{Z-Image-Turbo}
  & W4A4 & $0.998 \pm 0.002$ & $0.880 \pm 0.009$ & $0.766 \pm 0.022$ & $0.875 \pm 0.015$ & $0.464 \pm 0.017$ & $0.618 \pm 0.020$ & $\mathbf{0.767 \pm 0.001}$ \\
  & W2A4 & $0.616 \pm 0.149$ & $0.165 \pm 0.045$ & $0.243 \pm 0.090$ & $0.433 \pm 0.086$ & $0.100 \pm 0.037$ & $0.103 \pm 0.029$ & $\mathbf{0.276 \pm 0.072}$\\
\bottomrule
\end{tabular}}
\end{table*}

\begin{table*}[t]
\caption{Video-generation results on Wan~2.1-1.3B and CogVideoX-2B at W4A4. Scores are percentages. P/Q marks each method as quantization-aware training (QAT) or post-training quantization (PTQ). \textbf{Bold} and \underline{underlined} indicate the best and second-best result across all methods within each model; full-precision rows are references.}
\label{tab:vbench-qat}
\centering
\resizebox{\textwidth}{!}{%
\begin{tabular}{clcccccccccc}
\toprule
Model & Method & P/Q & Bit & \makecell{Imaging\\Quality~$\uparrow$} & \makecell{Aesthetic\\Quality~$\uparrow$} & \makecell{Motion\\Smoothness~$\uparrow$} & \makecell{Dynamic\\Degree~$\uparrow$} & \makecell{Background\\Consistency~$\uparrow$} & \makecell{Subject\\Consistency~$\uparrow$} & Scene~$\uparrow$ & \makecell{Overall\\Consistency~$\uparrow$} \\
\midrule
\multirow{11}{*}{Wan~2.1-1.3B}
  & Full Prec.       & --  & 16/16 & 64.30 & 58.21 & 97.37 & 70.28 & 95.94 & 93.84 & 28.05 & 24.67 \\
  \cdashline{2-12}\\[-1.75ex]
  & LSQ~\cite{lloyd1982least}              & QAT & W4A4 & 59.11 & 49.09 & \textbf{98.35} & 71.11 & 92.66 & 91.67 & 10.38 & 18.83 \\
  & Q-DM~\cite{li2023qdm}                  & QAT & W4A4 & 60.40 & 52.50 & 97.22 & \underline{76.67} & 93.37 & 89.26 & 13.28 & 21.63 \\
  & EfficientDM~\cite{he2024efficientdm}   & QAT & W4A4 & \underline{60.70} & \underline{53.57} & 96.18 & 56.39 & 93.74 & 91.70 & 11.77 & 21.19 \\
  & QVGen~\cite{huang2025qvgen}            & QAT & W4A4 & \textbf{63.08} & \textbf{54.67} & \underline{98.25} & \textbf{77.78} & \textbf{94.08} & \underline{92.57} & \underline{15.32} & \underline{23.01} \\
  \cdashline{2-12}\\[-1.75ex]
  & SmoothQuant~\cite{xiao2023smoothquant} & PTQ & W4A4 & 46.32 & 36.33 & 96.39 & 51.94 & \underline{95.85} & 90.39 & 2.79 & 15.05 \\
  & QuaRot~\cite{ashkboos2024quarot}       & PTQ & W4A4 & 51.42 & 40.49 & 96.21 & 52.78 & 95.76 & 88.80 & 5.31 & 17.98 \\
  & ViDiT-Q~\cite{zhao2024vidit}           & PTQ & W4A4 & 44.51 & 36.43 & 96.16 & 58.06 & \textbf{95.92} & 89.59 & 1.85 & 13.11 \\
  & SVDQuant~\cite{li2024svdquant}         & PTQ & W4A4 & 57.57 & 46.30 & 94.21 & \underline{72.22} & 93.16 & 77.96 & 12.73 & 21.91 \\
  & \cellcolor{gray!25}OrbitQuant          & \cellcolor{gray!25}PTQ & \cellcolor{gray!25}W4A4 & \cellcolor{gray!25}58.58 & \cellcolor{gray!25}53.41 & \cellcolor{gray!25}97.42 & \cellcolor{gray!25}53.89 & \cellcolor{gray!25}95.30 & \cellcolor{gray!25}\textbf{92.98} & \cellcolor{gray!25}\textbf{18.81} & \cellcolor{gray!25}\textbf{23.86} \\
\midrule
\multirow{11}{*}{CogVideoX-2B}
  & Full Prec.       & --  & 16/16 & 59.15 & 54.49 & 97.43 & 67.78 & 94.79 & 92.82 & 36.24 & 25.06 \\
  \cdashline{2-12}\\[-1.75ex]
  & LSQ~\cite{lloyd1982least}              & QAT & W4A4 & \underline{58.73} & \underline{54.20} & 97.57 & 45.00 & 92.97 & \underline{92.41} & 24.06 & 23.17 \\
  & Q-DM~\cite{li2023qdm}                  & QAT & W4A4 & 54.96 & 52.71 & 98.00 & \underline{48.61} & 93.82 & 91.86 & 28.02 & 23.87 \\
  & EfficientDM~\cite{he2024efficientdm}   & QAT & W4A4 & 55.96 & 51.97 & \underline{98.03} & 46.67 & \underline{94.10} & 91.70 & 27.76 & \underline{24.28} \\
  & QVGen~\cite{huang2025qvgen}            & QAT & W4A4 & \textbf{60.16} & \textbf{54.61} & \textbf{98.06} & \textbf{67.22} & \textbf{94.38} & \textbf{93.01} & \textbf{31.42} & \textbf{24.61} \\
  \cdashline{2-12}\\[-1.75ex]
  & SmoothQuant~\cite{xiao2023smoothquant} & PTQ & W4A4 & 39.90 & 35.50 & 97.80 & 1.90 & 95.90 & 92.90 & 3.60 & 12.80 \\
  & QuaRot~\cite{ashkboos2024quarot}       & PTQ & W4A4 & 49.60 & 47.10 & 96.90 & 9.20 & 94.80 & 90.20 & 19.70 & 21.70 \\
  & ViDiT-Q~\cite{zhao2024vidit}           & PTQ & W4A4 & 44.80 & 42.10 & 97.30 & 4.40 & 95.60 & 90.30 & 10.50 & 18.40 \\
  & SVDQuant~\cite{li2024svdquant}         & PTQ & W4A4 & 51.60 & 49.40 & 97.69 & 42.22 & 94.03 & 91.78 & 25.67 & 22.89 \\
  & \cellcolor{gray!25}OrbitQuant          & \cellcolor{gray!25}PTQ & \cellcolor{gray!25}W4A4 & \cellcolor{gray!25}52.62 & \cellcolor{gray!25}51.66 & \cellcolor{gray!25}96.99 & \cellcolor{gray!25}42.78 & \cellcolor{gray!25}94.50 & \cellcolor{gray!25}91.65 & \cellcolor{gray!25}\underline{28.53} & \cellcolor{gray!25}23.86 \\
\bottomrule
\end{tabular}}
\end{table*}

\subsection{Video Generation on Huge Model}
We evaluate the two largest video diffusion transformers in our study, Wan~14B~\cite{wan2025wan} and HunyuanVideo~\cite{kong2024hunyuanvideo}, at W4A4. Table~\ref{tab:vbench-wan14b} reports VBench per-dimension scores.
On Wan~14B, OrbitQuant is the best PTQ on seven of the eight dimensions and stays within noise of BF16 on most. The rotation and smoothing baselines drop sharply on the motion and scene dimensions, where the activation outliers are largest. On HunyuanVideo we compare against DVD-Quant~\cite{li2025dvd}, a quantizer designed specifically for the video model, with all baseline and DVD-Quant numbers taken from the DVD-Quant paper. Although OrbitQuant is calibration-free and uses one recipe across all backbones, it is competitive with DVD-Quant, ahead of it on imaging quality, motion smoothness, and scene. This suggests the distributional codebook still transfers to a model it was never tuned for, retaining usable quality without any per-model design.

\subsection{Robustness}
OrbitQuant is calibration-free, so the only stochastic parts of the pipeline are the RPBH rotation and the sampling noise. To confirm the results are not an artifact of a single random draw, we rerun the full pipeline with three random seeds and report the mean and standard deviation of GenEval. Table~\ref{tab:robustness} covers the three image models at W4A4 and W2A4.
At W4A4 the Overall standard deviation is at most 0.005 on every model, so a single seed is representative. The FLUX models stay similarly stable at W2A4, within 0.013 on Overall, and only Z-Image-Turbo at W2A4 shows larger variance. The analytic codebook and the random rotation otherwise give consistent results across seeds.

\subsection{Comparison with QAT}
\label{sec:supp-qat}
Table~\ref{tab:vbench-qat} places OrbitQuant alongside QAT methods that fine-tune the quantized model. As a calibration-free PTQ method, OrbitQuant is generally a step below the strongest QAT baseline QVGen~\cite{huang2025qvgen}, whose fine-tuned objective targets video quality directly. Even so, it stays close across most VBench dimensions and surpasses every QAT method on several, leading on Subject Consistency, Scene, and Overall Consistency on Wan~2.1-1.3B. That a fine-tuning-free quantizer matches or beats QAT on part of the benchmark, without any gradient step or per-model design, shows how strong OrbitQuant's rotated, calibration-free design is.

\section{Limitations and Future Work}
\label{sec:limit}
OrbitQuant inherits the online rotation that comes with rotation-based quantization. Unlike weight-only or BF16 inference, it applies the RPBH to activations at every forward pass, so a runtime rotation cost accompanies the memory savings, though this cost is small. We compute the per-block Hadamard transform with the fused \texttt{fast\_hadamard\_transform} kernel together with the random permutation, running at 0.451\,s per image on a single H100 at $1024^2$. This is an implementation limitation rather than a method one. Integer tensor cores compute on uniform grids, where the matmul runs directly on the quantized codes, while Lloyd--Max centroids are non-uniform, so no off-the-shelf kernel computes a codebook GEMM. Our current path therefore dequantizes codes and runs the matmul in BF16, as do all baselines under fake quantization. Lookup-table GEMM kernels for non-uniform weight quantization~\cite{park2024lut, guo2024fast} suggest the path forward, and we will build a fused kernel that gathers centroids in the GEMM prologue and computes in a native low-bit format.

\end{document}